\documentclass{article}
\usepackage{enumitem}
\usepackage[margin=1.25in]{geometry}
\usepackage{graphicx} 
\usepackage{amsmath}
\usepackage{amsfonts}
\usepackage{amssymb}
\usepackage{booktabs}
\usepackage{adjustbox}
\usepackage{cellspace}
\usepackage{wrapfig}
\usepackage{amsthm}
\usepackage{mathtools}
\usepackage{csquotes}
\usepackage{array}
\usepackage{url}
\usepackage{natbib}
\usepackage[dvipsnames]{xcolor}
\usepackage{algpseudocode}
\usepackage{algorithm}
\usepackage{soul}
\usepackage{parskip}

\newtheorem{theorem}{Theorem}

\newtheorem*{explanation}{Explanation}

\DeclareMathOperator*{\argmax}{arg\,max}
\DeclareMathOperator*{\argmin}{arg\,min}

\newif\ifshowcomments
\showcommentstrue

\ifshowcomments
  
  \newcommand{\eric}[1]{{\color{red}{\bf\sf [Eric: #1]}}}
  \newcommand{\ericm}[1]{\marginpar{\color{red}\tiny{Eric: #1}}}
  \newcommand{\mingkaitext}[1]{{\color{RoyalBlue}#1}}
  \newcommand{\mingkai}[1]{{\color{orange}{\bf\sf [Mingkai: #1]}}}
  \newcommand{\mingkaim}[1]{\marginpar{\color{orange}\tiny{Mingkai: #1}}}
  \newcommand{\jinyutext}[1]{{\color{PineGreen}#1}}
  \newcommand{\jinyu}[1]{{\color{blue}{\bf\sf [Jinyu: #1]}}}
  \newcommand{\jinyum}[1]{\marginpar{\color{blue}\tiny{Jinyu: #1}}}
\else
  
  \newcommand{\eric}[1]{}
  \newcommand{\ericm}[1]{}
  \newcommand{\mingkai}[1]{}
  \newcommand{\mingkaim}[1]{}
  \newcommand{\jinyutext}[1]{}
  \newcommand{\jinyu}[1]{}
  \newcommand{\jinyum}[1]{}
\fi

\title{Critique of Agent Model}
\author{
  Eric Xing\textsuperscript{$\diamond$,†}\thanks{~~Co-first author}\ ,
  Mingkai Deng\textsuperscript{$\diamond$,†}$^*$,
  Jinyu Hou\textsuperscript{$\diamond$,†}
  \\~\\
  \textsuperscript{$\diamond$}Institute of Foundation Models, Mohamed bin Zayed University \\ of Artificial Intelligence \\
  \textsuperscript{†} School of Computer Science, Carnegie Mellon University
  \\~\\
  \{eric.xing, mingkai.deng, jinyu.hou\}@mbzuai.ac.ae
}
\date{June 15, 2026}
\begin{document}

\maketitle

\begin{abstract}
What is an agent? What constitutes agency?
With the rise of Large Language Model (LLM) systems marketed as ``coding agents'', ``AI co-scientists'', and other ``agentic" tools that promise to drive up productivity, and at the same time, ``existential" concerns such as AI escaping human control with destructive power under a speculative ``machine agency" against humans, 
it has become essential to clarify where automation ends and agency begins, both for building capable systems and for understanding whether and what to fear.
Drawing on Descartes' grounding of agency in independent thought, and on portrayals of autonomous beings in science fiction, we survey the current landscape of AI agents, and analyze agent architectures along five dimensions: goal, identity, decision-making, self-regulation, and learning.
Specifically, we argue that genuine agency requires these structures to be \emph{internalized within the system itself} rather than assembled through external scaffolding.
This distinction between \emph{agentic} systems, whose competence resides in engineered workflows, and \emph{agentive} systems, whose capabilities (including social interaction) arise endogenously, defines the boundary between systems designed for prescribed tasks, and those capable of operating in the open world with true autonomy. 
Building on this analysis, we propose the Goal-Identity-Configurator (GIC) architecture for a general-purpose agent model, combining hierarchical goal decomposition, identity evolution, simulative reasoning grounded in a separately trained world model, learned self-regulation, and self-directed learning from both real and simulated experience.
Furthermore, we share insight on the auditability, controllability, and safety of agentive systems that possess greater autonomy and ``agency", but remain under human oversight.   

\end{abstract}

\section{Introduction}
\label{sec:introduction}

What is an agent? What constitutes genuine agency? 
For centuries, the question of human agency has been central to philosophy, psychology, sociology, and economics. 
Across these traditions, agency has been associated with properties such as long-term goals, evolving identity, purposeful planning, formation of social relationships, self-regulation, self-reflection, all the way toward moral responsibility and free will. 
Philosophical accounts, from Aristotle's discussions of purposeful action~\cite{aristotle2009nicomachean} to later views by Descartes~\cite{descartes1641} that thinking defines existence (\emph{“Cogito, ergo sum"}), suggest that 
agents are not just static entities that respond to external stimuli, but dynamic individuals with the ability to reason independently and act freely but rationally in pursuit of goals and well-being.

Can such biologically rooted agency be realized through artificial and mechanical means? 
A familiar illustration of autonomous artificial agents appears in science fiction. \emph{Blade Runner}~\cite{scott1982bladerunner}, a genre-defining classic, portrays \emph{replicants}, a type of bio-engineered beings that rival or surpass humans in strength, agility, and intelligence. These replicants are by no means perfect: they experience confusion, make mistakes, and suffer harm. Yet they possess human-like bodies, read and speak, move and work in the physical world, form deep inter-agent bonds, and in some cases question their own sense of self. Eventually, some bravely step out of their assigned roles towards a future of uncertainty and freedom. 
Such thought experiments highlight that agency is not synonymous with operational excellence (although often called for), but instead 
involves the capacity for goal-directed actions, self-development, self-reflection, participation in complex social environments, 
and, ultimately, possession of free will, morality, and a drive for self actuation.

This deeper notion of agency stands in contrast to many modern systems labeled as ``agents'' in contemporary AI research and development. 
These systems are capable of executing complex tasks (e.g., software engineering, computer use, dance performance) through carefully engineered scaffolding, including predefined tools, workflows, and programmatic control loops that guide behavior through 
externally defined
structure~\citep[e.g.,][]{anthropic_claude_code,openclaw_openclaw,boston_dynamics_spot_2026}. 
While these systems have achieved impressive practical success, their capabilities largely arise from orchestrating predefined workflows within constrained environments. In many cases, behaviors are determined by externally specified tools, protocols or training processes~\citep[e.g.,][]{anthropic_mcp_2024,anthropic_agent_skills_2025,slime_github}, rather than by an endogenous,  flexible decision-making process and intrinsic will. 

We find it useful to distinguish between two levels of autonomous systems. \textbf{Agentic} systems, such as those described earlier, complete tasks autonomously through orchestrated tools and workflows; their competence resides primarily in the engineering around 
a given reasoning model such as a LLM.
\textbf{Agentive} systems, exemplified by biological agents and discussed at length in this paper, possess agency in the fuller sense:
they derive their capabilities \emph{endogenously} (e.g., maintaining long-term goals, evolving self-identity, simulating future possibilities, regulating when and how to reason, or learning better behaviors) rather than following prescribed procedures, whether at \textbf{inference time} (e.g., fixed planning-execution workflows) or across the \textbf{development lifecycle} (e.g., manual training--deployment--retraining cycles). 
Current AI systems are largely agentic but not yet agentive: much of their competence resides in their workflows and harnesses, not in the model itself. Consequently, such systems are often better understood as sophisticated software pipelines rather than genuinely autonomous agents.
While these systems represent meaningful progress, they address only a portion of the broader challenge of 
artificial 
agency.

Indeed, it is difficult to imagine how enumerating every possible behavior through tools, prompts, or skills will allow AI systems to scale to the diversity and adaptability observed in biological agents. 
Humans, for example, exhibit multiple tiers of intelligence (Figure~\ref{fig:human-capabilities}): linguistic and symbolic reasoning (e.g., reading, writing, coding), physical and spatial competence (e.g., navigation, manipulation), social understanding (e.g., coordinating and competing with other agents), and higher-level ``philosophical'' capacities (e.g., curiosity, self-reflection, and goal formation). 
A single cognitive architecture is able to support this broad range of behaviors without requiring explicit re-engineering for each new task. 

\begin{figure}
    \centering
    \includegraphics[width=\linewidth,page=2]{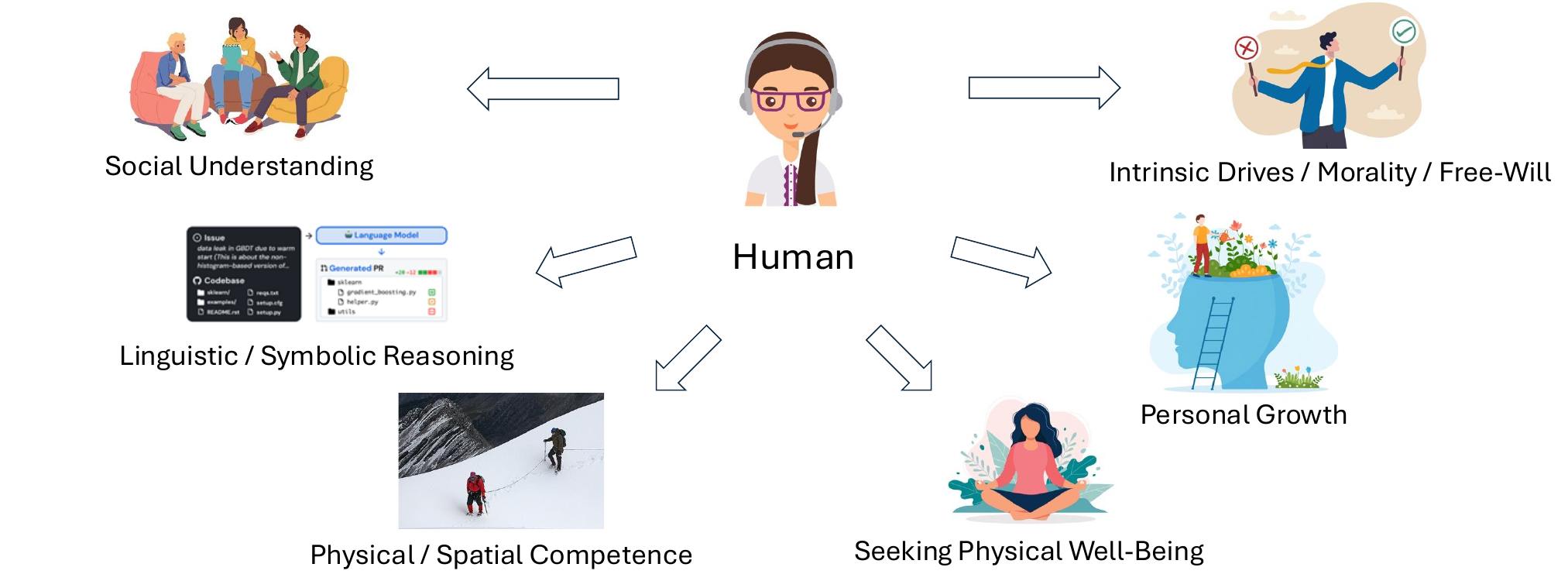}
    \caption{Humans exhibit multiple layers of intelligence: linguistic and symbolic reasoning, physical and spatial competence, social understanding, and higher-level ``philosophical'' capacities.}
    \label{fig:human-capabilities}
\end{figure}

Motivated by this observation, we argue that agency should not be treated as the accumulation of external scaffolding, but rather as a property emerging from a model capable of developing its identity, pursuing goals, and expressing and organizing its behavior across diverse environments. 
Rather than constructing agents through increasingly complex software pipelines, we study the problem of \emph{modeling agency itself}: developing machine learning models capable of generating a broad range of actions with the flexibility, adaptability, and autonomy associated with natural agents (e.g., humans and other animals), and of learning autonomously and perpetually. We refer to such a model as an \textbf{Agent Model}.
Specifically, an agent model (AM) is a reasoning model that generates real-world actions based on its goals $g$ and identity $i$. 
Formally, an AM $\pi$ maps the current world state $s$ to a predicted action $a$ through, for example, a conditional probability distribution:
$$p_\pi(a \mid s, g, i).$$
Equipped with such a model, a machine can draw on conceptual knowledge and logical/mathematical reasoning for abstract problem-solving, as well as act in the physical world via its end actuators (e.g., a humanoid body). 
Crucially, conditioning on goal $g$ and identity $i$ enables the system to \textbf{inspect, decompose, and revise} its long-term objectives (e.g., self-preservation or safety constraints) and self-model (e.g., capabilities and roles) rather than leaving them implicitly distributed across model weights and thus difficult to modify. 
Whether these are kept fixed by design or updated dynamically is a hallmark of the distinction between \emph{agentic} and \emph{agentive} systems. 
Similarly, how the model $\pi$ selects actions and updates itself reflect the key differences: \emph{agentic} systems follow fixed decision-making procedures and require externally scheduled training to improve, while agentive ones \textbf{regulate its own} deliberation mode during inference (e.g., reacting immediately to emergency vs. planning carefully for a complex maneuver) and capability updates during learning (e.g., retreating into simulated practice to address an identified weakness). 
Agency, in this view, arises from intentional actions generated by the model itself rather than from passively following externally scaffolded instructions. We discuss these distinctions in more detail in \S\ref{sec:boundary}.
How, then, should such a model be built? 
A basic principle, which we discuss formally in \S\ref{subsec:critique-decision-making} and \S\ref{subsec:critique-learning}, is that the agent model must be kept functionally distinct from a world model~\cite{xing2025critiques}: the former decides what \emph{to do}, the latter predicts what \emph{will happen}. Collapsing both into a single model, as several recent proposals do~\cite{ye2026dreamzero,li2026functional_taxonomy_world_models,nvidia2026cosmos3}, conflates reward-driven action selection with fidelity-driven next-state prediction, undermining the reliability of both planning and simulation. 
At a high-level, constructing and training an Agent Model involves five key aspects: 
\textbf{goal, identity, decision-making, self-regulation, and learning}.
The past two years have seen an explosion of systems labeled as agents, accompanied by competing schools of thought on how such systems should be designed. 
Proposals for addressing some of the aforementioned aspects leading to an agent model were offered in these attempts, but a systemic treatment of all aspects with a single framework possible for implementation is still unavailable.
In this paper, we categorize these approaches and analyze their limitations towards scalable and general-purpose agency.
Based on such, we introduce the \textbf{GIC} (Goal-Identity-Configurator) architecture, 
which provides
concrete proposals for 
each of the five aspects
of 
artificial agency and resultant capabilities
within a single adaptive system, paired with a separately learned world model.
Specifically, the GIC architecture combines: 1) \underline{hierarchical goal decomposition} with persistent objectives; 2) an \underline{evolving identity} that adapts without needing retraining; 3) \underline{simulative planning} through an internal world model (System~II) alongside reactive action (System~I); 4) \underline{self-regulation} of when and how deeply to deliberate via a learned configurator (System~III); and 5) \underline{self-directed learning} from both real and simulated experience. We present these ideas in detail in the sections that follow.
\section{The Boundary Between Agentic and Agentive Systems}
\label{sec:boundary}


Having introduced the distinction between agentic systems, which complete tasks through externally orchestrated tools and workflows, and agentive systems, whose capabilities arise from internal organization, we now formalize the dimensions along which they differ. Our goal is not to dismiss existing agentic systems, but to identify the minimal properties required for genuine agency, as a guideline for inspiring plausible design and implementation. Each dimension below defines a spectrum: at one end, the relevant structure is fully prescribed by external engineering; at the other, it is maintained and revised internally by the agent as part of its own decision-making. 

\subsection{Preliminaries: Agent-Environment Model}
\label{subsec:agent-environment-model}
\begin{figure}
    \centering
    \includegraphics[width=0.85\linewidth]{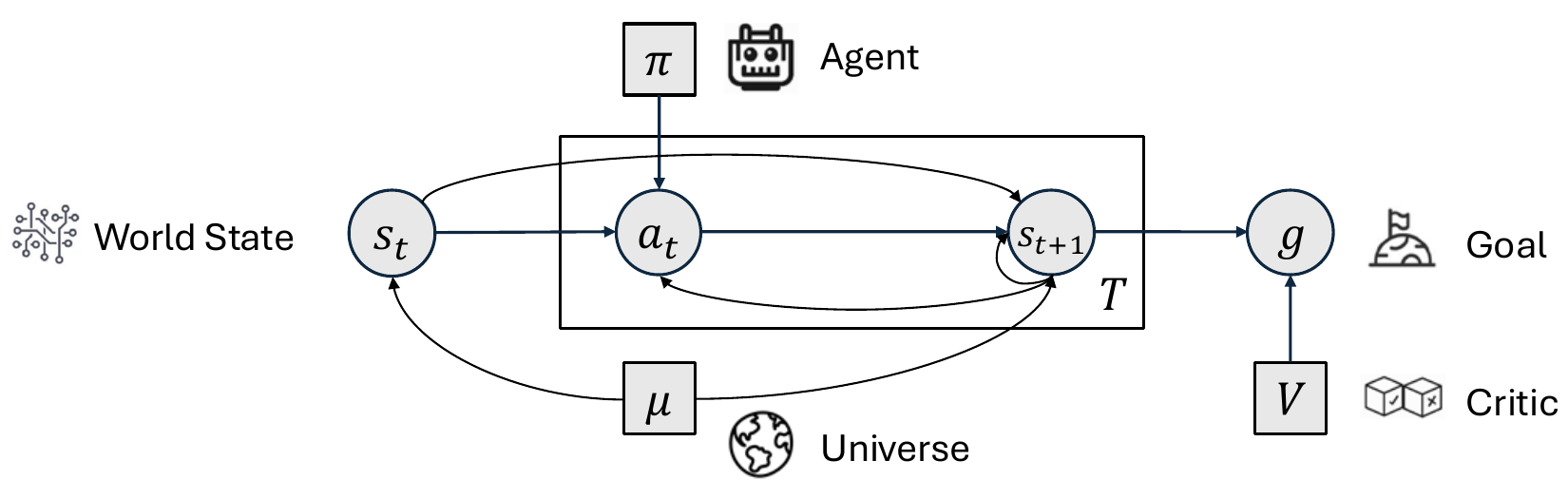}
    \caption{Illustration of an agent acting in an environment to achieve a goal.}
    \label{fig:optimal-agent}
\end{figure}

We begin with a minimal formulation of sequential decision making as a neutral foundation for the discussion that follows. Consider an environment (or \emph{universe}) represented by a stochastic dynamical system $\mu$, encompassing virtual, physical, and social components. The environment evolves over discrete time steps indexed by $t$ (continuous timesteps can be approximated by infinitesimally small discrete steps). 
Let $s_t$ denote the world (and internal) state at time $t$ and $a_t$ an action. The environment defines a transition distribution $p_\mu(s_{t+1} \mid s_t, a_t)$, and an agent is modeled as a policy $\pi$ that produces an action distribution $p_\pi(a_t \mid s_t)$. Given an initial state $s_t$, the interaction between $\pi$ and $\mu$ induces a trajectory distribution:
\begin{align}
  p^{\pi}_{\mu}(a_t, s_{t+1}, \dots, a_{T-1}, s_T \mid s_t)
  =
  \prod_{k=t}^{T-1}
  {\underbrace {\textstyle p_\pi(a_{k} \mid s_k)}_{\text{ agent }} } \ {\underbrace {\textstyle p_\mu(s_{k+1} \mid s_k, a_{k})}_{\text{ universe }} }.
  \label{eq:trajectory}
\end{align}
Equation~\ref{eq:trajectory} describes observable interaction dynamics without assuming any particular internal structure of the agent. 
The factorization also decomposes the subject of our discussion into exactly two objects: the \emph{agent} factor $p_\pi(a_k \mid s_k)$, which decides what \emph{to do}, and the \emph{universe} factor $p_\mu(s_{k+1} \mid s_k, a_k)$, which determines what \emph{happens next}. An \textbf{agent model} (AM) is a learned realization of the former; a \textbf{world model} (WM) is a learned approximation of the latter. 

We note that the term ``world model'' has recently been used more broadly, encompassing not only next-state prediction but also next-action generation~\cite{ye2026dreamzero,li2026functional_taxonomy_world_models,nvidia2026cosmos3}, in effect collapsing the two factors of Equation~\ref{eq:trajectory} into a single object. Throughout this paper, we keep them distinct: ``world model'' refers strictly to the universe factor, and ``agent model'' to the agent factor together with the internal structures, introduced below, that realize it. We believe the absence of a clear, functional definition of the agent model, distinct from the world model, may have contributed to action generation being absorbed into world-model frameworks by default; this paper offers one such definition and explores its consequences for how the agent reasons (\S\ref{subsec:critique-decision-making}, \S\ref{subsec:the-gic-arch}), why the two models call for different training signals (\S\ref{subsec:critique-learning}, \S\ref{sec:am-training}), and how failures are diagnosed and corrected (\S\ref{sec:am-safety}).

In the following subsections, we construct an agent model by introducing latent variables (goals, identity, plans, and regulation mechanisms) that formalize the properties of 
\emph{endogenous}
agency outlined above.
While goals and identity could also be viewed as components of the world state observable by other agents (e.g., one agent inferring another's goals from its behavior), we model them here as latent variables internal to the agent, since our focus is on the degree to which these structures are endogenously maintained vs. externally prescribed.

\subsection{Goals and Subgoals}
\label{sec:goals-subgoals}
We first enrich the agent-environment formulation by introducing \emph{goals}, which represent desired outcomes guiding decision-making over time. We denote the agent's goal at time $t$ by a latent variable $g_t$, conditioning action selection as $p_\pi(a_t \mid s_t, g_t)$. As with the other dimensions discussed below, we distinguish two limiting cases. 
On one end are externally specified goals, where objectives $g_t$ are supplied at each step (e.g., user instructions, prompts, or task specifications) and disappears once the interaction ends. 
On the other end are internally persistent goals $g$, which remain consistent over long horizons. 
An agent with persistent goals $g$ interprets immediate tasks not as its entire objective, but as subgoals $g_t$ within a larger, continuing trajectory of behavior. 
In this view, responding to individual user instructions is equivalent to having the top-level goal of ``satisfy external directions'', with the subgoals as each instruction. The agent's capacity, however, extends beyond this special case:
It may decompose a long-term goal $g$ into a sequence of subgoals $(g_1, g_2, \dots)$, ordered by dependency and priority, and revisable as new information arrives:
\[
    g_t \sim p_\delta(\cdot \mid s_t, g).
\]
This hierarchical structure isolates the difficulty of long-horizon planning in the decomposition module $\delta$, while each subgoal $g_t$ can be pursued by short-horizon capabilities that are easier to learn and supervise. 
A common way to evaluate goal-directed behavior is through a reward function $r(s_t, g_t)$ measuring the compatibility between the current state and the agent's current subgoal, and the long-term performance of a policy is evaluated by the expected discounted cumulative reward, also known as the value function~\cite{sutton1998reinforcement}, with the discount parameter $\gamma_t$ satisfying $\lim_{t \rightarrow \infty} \gamma_t = 0$:
\begin{align}
    V_{\pi,\mu}^{g_t}(s_t) &\vcentcolon= \mathbb{E}_{\pi,\mu} \left[ \sum_{k=t}^\infty \gamma_k r(s_k, g_t) \mathrel{\bigg|} s_t \right] \nonumber \\
    &= \lim_{T \rightarrow \infty} \sum_{(a_t, s_{t+1},\dots, s_T)} {\underbrace {\textstyle \sum_{k=t}^T \gamma_k r(s_k, g_t) }_{ \text{goal} } } \ {\underbrace {\textstyle p^{\pi}_{\mu}(a_t, s_{t+1}, \dots, s_T \mid s_t)}_{ \text{trajectory} } }
    \label{eq:value-function}
\end{align}
The degree to which goal formation, decomposition, and maintenance are endogenous to the agent is one axis along which agentic systems become agentive. Agentic systems largely execute externally specified instructions; agentive systems maintain, decompose, and revise their own goals as part of their ongoing decision-making.


\subsection{Identity}
We next introduce \emph{identity}: a latent variable $i_t$ capturing persistent properties that influence decision-making across time, such as capabilities, constraints, affordances, and relationships with other entities. Identity conditions action selection as $p_\pi(a_t \mid s_t, g_t, i_t)$, separating internal self-knowledge from observable dynamics.
A key question is how identity is maintained. At one end, identity is static: $i_t = i_0$ for all $t$, fixed by system design (e.g., system prompts, configuration files, or predefined roles). Such designs are practical when the environment is well-understood and predictable, but adaptation requires external re-engineering rather than endogenous updating. At the other end, identity evolves with the environment and internal state $s_t$ through the transition $\iota$:
\[
    i_t \sim p_\iota(i_t \mid s_t, i_{t-1}).
\]
An agent with adaptive identity revises its self-model in response to success, failure, or environmental feedback, analogous to how a professional updates self-assessment over the course of a demanding day. 
Identity in this sense functions not merely as initialization but as an evolving latent state participating in ongoing decision-making: capabilities and role assumptions may be revised, new affordances may be discovered, and relationships with other entities may be updated based on observed interactions.
The degree to which identity is originated, maintained and revised internally is one axis along which notions of agency differ.

\subsection{Decision-Making}

Given goals and identity, an agent must select actions that account for future consequences. 
Beyond simple fully observable settings~\citep[e.g.,][]{silver2016mastering,silver2017mastering}, however, the agent does not have direct access to the true world state $s_t$. Instead, it receives observations $o_t$ and infers a \emph{belief state} $\hat{s}_t$ representing its best estimate of the world. A learned \textbf{world model} $f$ can then predict the next belief state given a proposed action, according to $p_f(\hat{s}_{t+1} \mid \hat{s}_t, a'_t)$. 
This $f$ is precisely a learned realization of the universe factor of Equation~\ref{eq:trajectory}, now operating in belief space: it remains a model of the world, distinct from the agent model that queries it.
By simulating sequences of actions and their predicted consequences, the agent can approximate optimal behavior without access to the true environment dynamics. Formally, the optimal policy under the world model $f$ selects action sequences that maximize expected goal progress under simulated state transitions, conditioned on the agent's current subgoal $g_t$ and identity $i_t$:
\begin{equation}
    \pi^*_f(\hat{s}_t, g_t, i_t) =
    {\underbrace{\argmax_{ a'_{t:T'-1} \in \mathcal{A}(i_t) }}_{\text{possible actions}}} \
    \sum_{\hat{s}_{t+1:T'}}
    \Bigg({\underbrace{ \sum_{k=t}^{T'-1} \gamma_k r(\hat{s}_k, g_t)
    + \gamma_{T'} V_{\pi, f}^{g_t}(\hat{s}_{T'}) }_{\text{goal progress}}}\Bigg)
    \prod_{j=t}^{T'-1} \
    {\underbrace{p_f(\hat{s}_{j+1} | \hat{s}_j, a'_j).}_{
    {\scriptsize \shortstack{simulation with\\world model}}}}
    \label{eq:world-model-decision-making}
\end{equation}
We refer to this form of deliberation as \textbf{simulative reasoning} (a form of System~II reasoning): the agent proposes candidate actions, predicts their consequences through the world model $f$, and selects the sequence that maximizes expected long-term progress. In contrast to traditional logical reasoning (e.g., deduction, induction, abduction), simulative reasoning provides a general-purpose planning mechanism grounded in verifiable next-state prediction, applicable across diverse tasks without domain-specific procedures~\cite{xing2025critiques}.


In practice, exact optimization over Equation~\ref{eq:world-model-decision-making} is intractable. We thus denote by $\pi_f$ a simulative planner that approximates $\pi^*_f$. Its output is a \emph{plan} $c_t$ encoding the current belief, a selected action sequence, and predicted future states:
\begin{equation}
   c_t = (\hat{s}_{t}, a'_{t}, \hat{s}_{t+1}, a'_{t+1}, \dots, \hat{s}_{T'})
   \sim p_{\pi_f}(\cdot \mid \hat{s}_t, g_t, i_t).
\label{eq:simulative-planning}
\end{equation}
The plan provides structured grounding for coherent behavior over long horizons: predicted future states can be checked against subsequent observations to assess plan validity, while planned actions guide execution when anticipated states are encountered or when the current state is highly uncertain (e.g., landing an airplane in low visibility). Given a plan $c_t$, the agent selects concrete actions through an \textbf{actor} $\alpha$ that handles fine-grained reactive execution: $a_t \sim p_\alpha(\cdot \mid \hat{s}_t, c_t)$. This reactive component (System~I) captures execution patterns that are difficult to encode in structured plans and enables fast response when deliberation is unnecessary.
The key distinction between \emph{agentic} and \emph{agentive} systems is therefore whether planning is an internal computational process (i.e., the agent forms, revises, and acts on plans as a result of its own decision-making) 
or an externally imposed procedure (e.g., forced reaction, predefined workflow, or always-on model-predictive control). A separate question is how the agent determines \emph{when} and \emph{how much} planning to perform, which we address next.


\subsection{Self-Regulation}
Long-horizon planning introduces a question beyond \emph{what} action to take: \emph{how} should the decision be made? Different situations call for different amounts and types of internal computation, depending on urgency, difficulty, uncertainty, and resource budget. Some decisions may be handled by direct policy execution (e.g., dodging a ball), while others benefit from extended deliberation or replanning (e.g., strategizing a full match).
More broadly, such meta-decisions also encompass whether to pursue or abandon a goal, whether to act or refrain from acting, and how to prioritize competing objectives, extending beyond computational resource allocation to behavioral and normative dimensions. We refer to the capacity to control these internal modes of operation as \emph{self-regulation}.
We model this through a \textbf{configurator} $\kappa$, which outputs a regulation variable $u_t$ governing the agent's decision mode at each step (e.g. whether to act directly, continue executing an existing plan $c_{t-1}$, invoke additional planning, or revise goals: 
\[u_t \sim p_{\kappa}(\cdot \mid s_t,\, g_t, i_t, c_{t-1}).\]
Self-regulation is thus itself part of the agent's policy: the allocation of internal effort adapts with experience rather than following fixed rules or designer-specified workflows. 
Furthermore, the configurator may extend beyond inference-time deliberation to govern the agent's own learning process (e.g., deciding when to act in the environment, when to retreat into simulation for practice, when to update its world model, and when to revise its self-model). We return to this point below.
The degree to which deliberation control is endogenous to the agent is another axis along which agentic systems are distinguished from agentive ones. Agentic systems follow externally prescribed workflows; agentive systems organize their own computation in response to changing circumstances.
 

\subsection{Learning}
The preceding subsections describe how an agent acts given its current capabilities. A separate question is how those capabilities improve over time. In most existing systems, learning terminates before deployment, and behavioral change thereafter requires external intervention such as retraining or prompt redesign.
A growing body of work addresses this limitation under labels such as 
``never-ending learning''~\cite{mitchell2018never},
``recursive self-improvement''~\cite{patel2026darioamodei} or ``auto research''~\cite{karpathy2026autoresearch},
which use AI systems to automate aspects of the traditional training pipeline (e.g., generating synthetic tasks and curricula, performing automated evaluation).
However, in virtually all such ``AI training AI'' systems, the learning process itself remains external to the agent, with training decisions (e.g., when to learn, what data to use, how long to train, and when to stop) ultimately made by the human engineer, not by the agent whose capabilities are being updated.
A more complete notion of agency, on the other hand, treats learning as continuous and endogenous, taking two complementary forms: \emph{learning from real interaction}, where the agent updates its parameters $\theta$ based on deployment experience, and \emph{learning from simulated experience}, where the agent generates hypothetical trajectories through its \textbf{world model} $f$ and trains on them without real-world interaction. Formally, we define $\lambda$ as the learning process that outputs the next parameter $\theta_{t+1}$ given current parameters $\theta_t$ and real and simulated experiences $D_\mu$ and $D_f$ as below:
\[
\theta_{t+1} \sim p_\lambda(\cdot \mid \theta_t, D_\mu, D_f).
\]
Simulative learning is particularly valuable when real-world trial-and-error is dangerous, expensive, or slow.
Note that the two models implicated here learn from different signals: the world model $f$ improves by reducing prediction error against observed transitions, while the agent's decision-making components $\theta$ improve through goal-directed feedback, a separation whose importance we argue in detail in \S\ref{subsec:critique-learning}.
Another key difference from current ``AI-builds-AI'' approaches is that in the self-directed agent, learning is governed by the configurator $\kappa$ as part of the agent's own policy, rather than being imposed on the agent as an external schedule. In addition to model parameters $\theta$, the self-model $i$ may also be updated in the manner discussed earlier, as a fast improvement procedure without needing full retraining.
The degree to which learning is internally initiated and regulated is another axis along which {agentic systems} differ from {agentive systems}. Current systems, even those that automate training with AI, are still \emph{agentic} as the training loop remains external and the agent remains frozen unless retrained. \emph{Agentive} systems, by contrast, improve autonomously and perpetually through experience, augmenting external interaction with internal world-model simulations, and governing its own learning as an integral part of its ongoing decision-making.



\subsection{Coordination and Communication}

In a social environment, an agent must often decide whether to communicate, whom to engage, what information to share, and how to interpret the behavior of others in light of their likely identities, capabilities, and goals. Communication and coordination thus emerge as autonomous decisions, arising from the agent's native communicative abilities, an environment composed of other agents, and tasks that require multi-agent interaction. Natural agents exhibit a further capacity for \emph{self-organization}: individuals form, revise, and dissolve patterns of coordination, without requiring those structures to be specified in advance. 
In practice, many existing systems construct ``multi-agent teams''~\cite{autogen} or ``agent swarms''~\citep[e.g.,][]{openai_swarm}, but these often externally specify the nature and pattern of interaction (e.g., team membership, communication protocols, role assignments, and coordination logic) via the human designer. Such systems are better understood as a single scaffolded system consisting of a federation of tasks rather than a genuine multi-agent society. 
As with the other dimensions, how multi-agent interaction is handled delineates the boundary between agentic and agentive systems: \emph{agentic} systems require orchestrating interaction patterns externally; \emph{agentive} systems allow collective organization to emerge as an internal decision of participating agents. 

The properties introduced above together characterize what genuine agency should minimally possess. 
The distinction between \emph{agentic} and \emph{agentive} systems is not simply about whether relevant structures (e.g., \emph{goals}, \emph{identity}) exist, but in how these behaviors originate: through externally engineered pipelines that prescribe behavior, or an internal \emph{configurator} capable of adapting, revising, and organizing their own decision-making processes (e.g., planning, self-regulation, learning, and interaction). 
This perspective motivates the remainder of the paper, where we first examine whether and where current agentic systems fall short of this vision~(\S\ref{sec:agent-model-landscape}-\ref{sec:critiques}), and then present the \textbf{Goal-Identity-Configurator} (GIC) agent model architecture where these structures arise as components of a single adaptive system, paired with a separately learned world model~(\S\ref{sec:gic-agent-model}).

\section{Landscape of Systems Labeled as ``Agents''}
\label{sec:agent-model-landscape}


The term ``agent'' is currently applied to a remarkably broad range of systems, from simple automation scripts to embodied learning systems. This breadth, however, obscures an important distinction highlighted in the previous section: systems may appear goal-directed while differing fundamentally in where the organization of behavior resides. 
Rather than organizing the landscape by application domain, we examine it through the mechanisms that produce behavior. 
This perspective reveals a continuum from systems whose competence is almost entirely prescribed by software structure, to systems that increasingly internalize planning, acting, and adaptation within a single model. 

\paragraph{Program-Based Systems and Classical Bots}
From the earliest days of computing, practitioners have built software systems that act toward explicit goals through deterministic logic~\citep{newell1976computer,davis1977overview}. 
A thermostat observes temperature and applies fixed control rules; ELIZA~\cite{weizenbaum1966eliza} simulates psychotherapy through pattern matching (with surprising effectiveness); browser automation frameworks like Selenium~\cite{selenium_webdriver} and Playwright~\citep{microsoft_playwright} execute scripted interaction sequences in digital environments. 
These systems can clearly pursue objectives, but every aspect of their behavioral organization (e.g., goals, identity, decision-making, adaptation) is fixed by design. 
From the perspective developed earlier, these are best understood as software pipelines, not internally organized agents.

\paragraph{LLM Wrapper Systems}
A large fraction of contemporary systems marketed as ``AI agents'' place pretrained LLMs inside structured orchestration layers, whether it be plan-search-read-synthesize loops (e.g., DeerFlow~\cite{bytedance_deerflow}), tool-calling pipelines (e.g., Agent Skills~\cite{anthropic_agent_skills_2025}), or multi-agent coordination graphs (e.g., AutoGen~\cite{autogen}), which specify how behavior should unfold.
Deployed instances span customer-service automation (e.g., Decagon~\cite{decagon_website}), coding assistants (e.g., Cursor~\cite{cursor_agents}), personal assistants (e.g., OpenClaw~\cite{openclaw_openclaw}), and scientific automation (e.g., CRISPR-GPT~\cite{qu2025crispr}).
Despite often impressive task competence, the LLM in these systems contributes flexible reasoning and instruction following, while the surrounding scaffold is responsible for structuring goals, specifying identity, orchestrating planning, and compensating for model weaknesses. 
The organization of behavior thus resides in the engineering around the model, not in the model's own decision-making. 
\paragraph{LLM-Centered Systems}
A more recent class of systems shifts more of the behavioral structure into the model itself, training or fine-tuning LLMs to map observations to actions over extended trajectories (often with chain-of-thought~\cite{wei2022chain}). 
One direction trains models end-to-end for specific domains, including browser use (e.g., OpenAI Operator~\cite{openai_computer_using_agent_2025}), deep research (e.g., Tongyi-DeepResearch~\cite{tongyidr}), software engineering (e.g., Claude Code~\cite{anthropic_claude_code}), and game playing (e.g., SIMA-2~\cite{bolton2025sima}). 
A second, increasingly active direction trains general-purpose agentic LLMs that integrate reasoning, tool-use, and multi-step interaction within a single model (e.g., DeepSeek-V4~\cite{deepseekai2026deepseekv4}).
Compared with wrapper systems, these approaches internalize more of reasoning and action selection, representing an important step toward fuller agency.
However, goals still depend on human-specified short-term instructions; identity remains externally defined; decision-making relies on unregulated chain-of-thought; and behavioral change still requires retraining or prompt redesign rather than self-directed learning from deployment experience. 

\paragraph{Model-less Physical Systems}
Embodied platforms are often intuitively associated with agency, but physical embodiment alone should not be confused with internally organized decision-making. 
Traditional industrial robots (e.g., ABB~\cite{abb_robotics_website}, FANUC~\cite{fanuc_robots_website}) execute carefully programmed routines, while modern legged autonomous platforms (e.g., Boston Dynamics~\cite{boston_dynamics_spot_2026}, ANYbotics~\cite{anybotics_anymal_2026}) typically combine learned low-level control with externally scripted task logic.
These systems may exhibit high physical competence while still relying on externally imposed task decomposition, action planning, and adaptation procedures. Embodiment therefore expands the action space, but does not by itself resolve the problem of agency.
\paragraph{Embodied-Model Systems}
The most ambitious current efforts aim to integrate perception, reasoning, and control into unified embodied models~\cite{fung2025embodied}. 
Generalist humanoid and manipulation platforms (e.g., Figure AI Helix~\cite{figure2025helix}, Physical Intelligence $\pi$ series~\cite{intelligence2025pi}) and autonomous driving systems (e.g., Waymo~\cite{waymo_driver_2026} and Alpamayo~\cite{wang2025alpamayo}) increasingly adopt vision-language-action (VLA) architectures trained from demonstrations, imitation learning, and large-scale simulation (e.g., NVIDIA Isaac Lab~\cite{nvidia_isaac_lab_2026}).
In parallel, world action models (WAMs; e.g., DreamZero~\cite{ye2026dreamzero}) jointly predict future states and actions within a shared architecture, incorporating aspects of world model into the policy itself.
These systems represent the closest current approximations to internally organized agents, acquiring physical priors from large-scale data and demonstrating generalization to unseen tasks and environments.
Nevertheless, these systems are still limited in their sensory repertoire (e.g., no force, texture, hardness, or temperature).
Important aspects of agency, such as goal decomposition, identity evolution, self-regulated deliberation, and self-directed learning are missing.
As such, training remains heavily dependent on expert demonstrations; no mechanism exists for the agent to modulate how much deliberation a given situation warrants; most systems remain confined to short-horizon tasks with limited capacity for sustained goal pursuit or open-ended coordination; and adaptation beyond the training distribution still requires external human intervention.

\paragraph{Relation to Existing Surveys}
Parts of the landscape above have been documented in several recent surveys.
Wang et al.~\cite{wang2023survey} systematize LLM-based agents organized by profiling, memory, planning, and action modules; Wei et al.~\cite{wei2026agentic} extend this scope across foundational, self-evolving, and collective reasoning layers; Jiang et al.~\cite{jiang2025adaptation} study post-pretraining adaptation under a unified framework; Gao et al.~\cite{gao2025survey} and Fang et al.~\cite{fang2025comprehensive} focus on mechanisms of continual adaptation; and Chu et al.~\cite{chu2026agentic} survey world models in the context of agency.
These surveys offer comprehensive coverage of what current systems can do and how they can be improved, but they tend to take the notion of agency itself for granted, treating it as a label that applies whenever an LLM interacts with an environment, rather than examining what structural properties a system must possess to warrant the designation.

Taken together, the landscape above shows that while recent systems have become remarkably capable, much of that progress has come from improving external orchestration, narrowing domains, and exploiting increasingly powerful foundation models within carefully engineered workflows.
In many cases, the core structures of agency, whether it be endogenous goal decomposition, persistent self-models,  adaptive self-regulation, continual learning, or autonomous social organization, still reside outside the model.
This observation motivates the central question of the next section: across the dimensions that distinguish genuine agents from software pipelines, \emph{where} exactly do current systems fall short, and \emph{what} would a model capable of internalizing these structures require?

\section{Critique of  Agent Modeling}
\label{sec:critiques}

As discussed in \S\ref{sec:agent-model-landscape}, the past two years have produced a remarkably diverse ecosystem of systems labeled as ``agents'', from GUI operators trained on screenshot-to-action trajectories, to coding assistants that thrive in verifiable repositories, to humanoid robots with dual-system control stacks. 
These systems frequently promise, and in some cases have already delivered, massive economic value, but remain limited in their pathways toward autonomous, generally applicable, and continuously improving agentive capabilities. 
In this section, we offer critical discussions on 
common practices in today's systems along 
the five axis of agency
identified in \S\ref{sec:introduction}: goals, identity, decision-making, self-regulation, and learning.
Each contention is followed by a constructive alternative describing what a more complete agent model requires. The resulting  proposal of a general architecture for agent models is presented thereafter in \S 5. 

Across the diverse systems surveyed in \S\ref{sec:agent-model-landscape}, a common design philosophy, which we shall dissect, has emerged, which can be summarized as follows:
\begin{enumerate}[leftmargin=12pt]
    \item \textbf{Goal}: Continuously supply the agent with short-term instructions $g_t$ from a human user (e.g., natural language prompt or target image), for easy and general controllability. 
    \item \textbf{Identity}: Specify the agent's capabilities, constraints, and affordances externally via fixed system prompts and/or configuration files; invest significant effort in \emph{harness engineering} for reliable and customizable execution.
    \item \textbf{Decision-Making}: Prioritize black-box, end-to-end policies, possibly with adaptive computation (e.g., chain-of-thought for LLMs and output queries for VLAs), and train them via reinforcement learning, 
    due to simplicity and end-to-end optimizability.
    \item \textbf{Self-Regulation}: Expect effective allocation of deliberation to emerge from unconstrained RL training, and/or build planning into fixed, human-designed workflow stages (e.g., plan-then-act pipelines, always-on model-predictive control), to enable controllable and predictable behavior. 
    \item \textbf{Learning}: Train the agent through human-scheduled pipelines (i.e., RL in rule-based simulators for safety and scalability, or supervised demonstration/correction in the real-world for downstream alignment), 
    to facilitate controllability and safety. 
\end{enumerate}


While these choices are often practical and produce capable systems, we argue that each introduces fundamental limitations toward scalable, general-purpose agency. 
Furthermore, as we will show, underlying those limitations is a common structural absence of an explicit internal model of reality: namely, a \textbf{world model} capable of predicting the consequences of actions in a given state, across layers such as mental, physical, social, and natural worlds. We will return to this observation at the end of the section, and begin by examining each of the limitations below.

\subsection{Goal: From Step-by-Step Instruction to Hierarchical Decomposition}
\label{subsec:critique-goal}

\begin{displayquote}
\textit{Continuously supply the agent with short-term goals $g_t$ at each step, for easy and general controllability 
-- not feasible for harder tasks.
}
\end{displayquote}

Contemporary agentic systems overwhelmingly operate with externally supplied, short-horizon goals. Coding assistants such as Claude Code~\cite{anthropic_claude_code} and Cursor~\cite{cursor_agents} receive task specifications for each operation; personal assistants such as OpenClaw~\cite{openclaw_openclaw} respond to individual user queries; vision-language models such as $\pi$-series~\cite{intelligence2025pi} and Helix~\cite{figure2025helix} condition on a target images or short instruction for each manipulation episode. 
In all cases, the system's objective disappears once the interaction ends, and a new goal must be supplied before behavior resumes. 

While this design yields controllable systems for short-horizon tasks (e.g., pick up a bottle), it is difficult to scale to tasks that demand higher levels of autonomy (e.g., make wine over a year's time). 
Indeed, as discussed in the distinction between scaffolded systems and genuine agency (\S\ref{sec:boundary}), a truly autonomous agent should be instructable with a long-term goal, not hand-held at every step. 
For goals that span extended time horizons (e.g., developing a drug candidate, conducting a multi-month research project, executing a complex logistics operation), demonstrations are rare and end-to-end RL by trial-and-error is prohibitively slow, making direct optimization over the full horizon impractical. 

The alternative is to take a hierarchical approach to modeling goals (Figure~\ref{fig:goal-decomposition}). Rather than requiring a human to supply every subgoal, the agent can include and learn a \textbf{goal decomposition module} $\delta$ that breaks down a long-term goal $g$ into a sequence of subgoals $(g_1, g_2, \dots)$, ordered by dependency and priority, and revisable as new information arrives (as formalized in \S\ref{sec:goals-subgoals}). 
This decomposition isolates the difficulty of long-term planning in $\delta$, while each subgoal $g_t$ can be executed by short-horizon capabilities that are easier to learn and supervise. 
The result is a form of hierarchical planning that allows the agent to tackle problems requiring extended courses of action, without requiring that the entire trajectory be optimized or supervised as a single monolithic episode. 
During inference and planning, effective decomposition itself can be treated as a decision-making task, which, as we argue in \S\ref{subsec:critique-decision-making}, benefits from simulating the consequences of proposed subgoals (e.g., achievability, ordering, dependencies) through a hierarchical world model $p_f(s_{t+T} \mid s_t, g_t)$ capable of simulating the long-term consequence $s_{t+T}$ after executing $g_t$ over multiple time steps. 


\begin{figure}
    \centering
    \includegraphics[width=\linewidth,page=3]{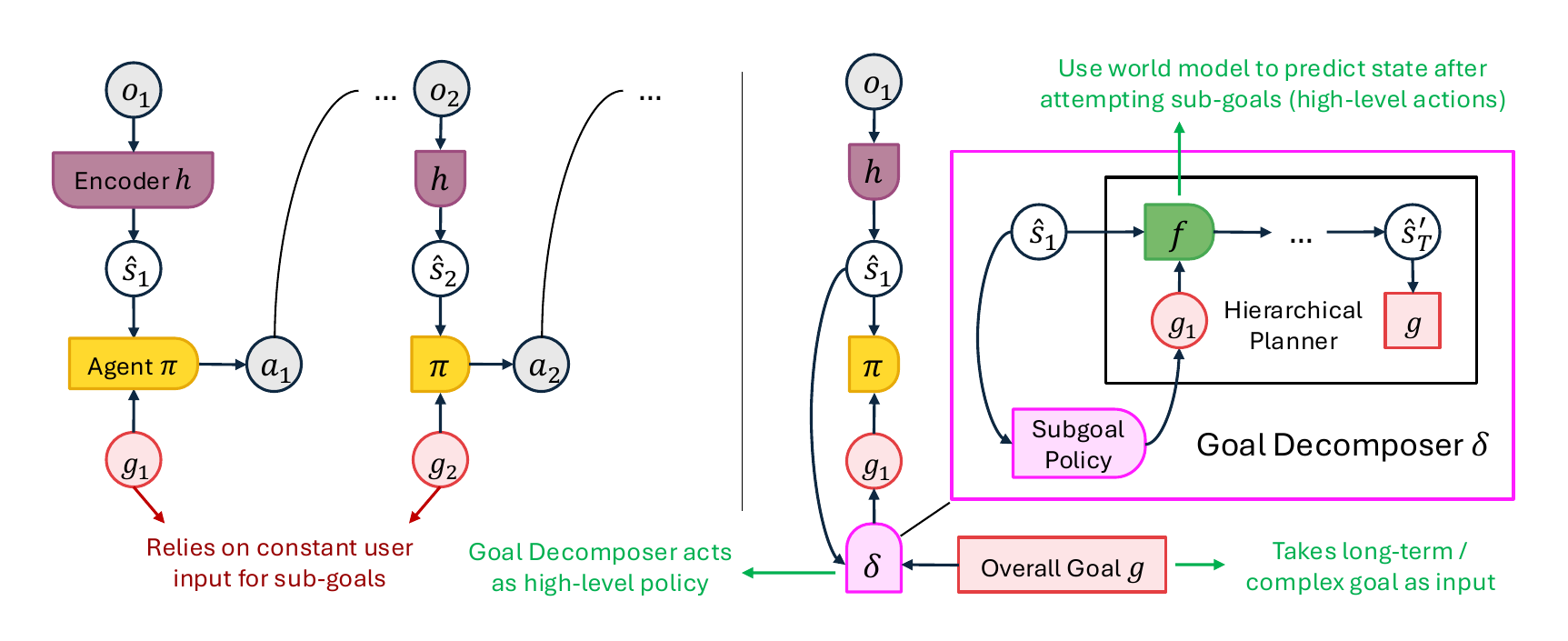}
    \caption{\small \textbf{Comparison of step-by-step subgoals to hierarchical decomposition of overall goal.} (\textbf{Left}) contemporary agentic systems are supplied a short-horizon goal $g_t$ at every step, and the objective disappears once the interaction ends. (\textbf{Right}) Alternative hierarchical approach instructs the system once with a long-term / overall goal $g$; a learned decomposition module $\delta$ breaks it into a sequence of subgoals $(g_1, g_2, \dots)$, selected based on outcomes predicted by a hierarchical world model $f$ and revised as the state $s_t$ evolves, each pursued by short-horizon capabilities that are easier to learn and supervise.}
    \label{fig:goal-decomposition}
\end{figure}

\subsection{Identity: From Harness Engineering to Adaptive Self-Models}
\label{subsec:identity-critique}

\begin{displayquote}
\textit{Specify the system's capabilities, constraints, and affordances externally via fixed system prompts or frozen latent vectors; invest in harness engineering for reliable and predictable behavior 
-- 
withholds full autonomy from the system.
}
\end{displayquote}

An agent's behavior is shaped not only by its goals and its model of the world, but also by what it knows about \emph{itself}: its capabilities, constraints, affordances, and relationships with other entities. 
Beyond the functional aspects, identity can even encompass broader dimensions such as values, loyalties, and moral commitments, which shape how an agent prioritizes and conducts itself in pursuit of its goals.
Just as the world model serves as the agent's theory of its environment, the self-model serves as its theory of its own mind. This distinction echoes Kant's separation of \emph{outer sense} (awareness of objects in the world) from \emph{inner sense} (awareness of one's own mental states)~\cite{kant1781}. 

Current practice, however, focuses on manual engineering to inform an agentic system about its capabilities, limitations, and how to use its tools. Identity is implemented as a hand-written system prompt describing the agent's role, available tools, and behavioral constraints. 
In systems built around tool-calling protocols such as MCP~\citep{anthropic_mcp_2024} and Agent Skills~\citep{anthropic_agent_skills_2025}, significant effort goes into ``harness engineering'' as advocated by OpenAI~\cite{lopopolo2026harness} and Anthropic~\cite{rajasekaran2026harness}: designing infrastructure that the agent can control, and describing that infrastructure to the agent in a way that maximizes effective use. 
In this case, the agent's self-model is specified externally and remains static. 
While designing strong interfaces for the agent is clearly valuable, current practice exogenizes what should be part of genuine agency: the formation and evolution of one's own identity. 
A fixed and/or externally specified identity cannot adapt when the agent encounters unexpected capabilities or limitations, especially when it is deployed in a new environment, or when it receives performance feedback that necessitates revision of its self-model. 
Without diminishing the value of well-designed infrastructure, the agent should be allowed to autonomously update its own understanding of its capabilities, constraints, and relationships based on experience, without requiring human re-engineering. 

The constructive solution draws on a \textit{fast--slow} update principle: rather than relying on a single adaptation mechanism, the agent maintains two 
complementary timescales of learning. \textit{Slow} updates modify model parameters $\theta_t$ (e.g. gradient-based training), which are computationally expensive, infrequent and more durable by design. \textit{Fast} updates revise a compact self-model $i_t$ more frequently during interaction, taking effect immediately without retraining, as formalized in Theorem~\ref{thm:fast-slow-dominates}.
This is analogous to how a professional revises self-assessment over a busy day without needing to constantly ``rewire their brain''. The intended effect is that the agent's behavior can reflect the most recent evidence about itself at any given moment, while slower parameter updates accumulate what has proven durable over longer horizons.
We show that, if fast updates in practice produce identity revisions that are better than random, the fast-slow agent learning accumulates strictly less regret in expectation than slow-only learning, and the gap widens with both the length of interaction and the number of update rounds.

\begin{theorem}[Fast-slow learning dominates slow-only learning, up to identity revision quality]
\label{thm:fast-slow-dominates}
Consider an agent operating over $K$ rounds, where each round $k$ consists of a slow update producing a base policy $\pi_k$, followed by $N_k$ steps of environmental interaction. In the slow-only setting, the agent acts under a fixed identity $i_0$ throughout each round. In the fast-slow setting, an identity evolver $\iota$ revises the self-model at each step, producing $i_t \sim p_\iota(\cdot \mid \hat{s}_t, i_{t-1})$.

Assume: (A1)~identity revisions improve the self-model, and better self-models produce better decisions; (A2)~the slow update operator is monotone in policy quality, both in the base policy it updates and in the data-generating policy. Then the fast-slow agent's cumulative regret satisfies:
\begin{equation}
    \emph{Regret}^{\emph{fast-slow}}_K \;\le\; \emph{Regret}^{\emph{std}}_K \;-\; \Omega\!\left(\textstyle\sum_{k=1}^{K} N_k\right),
    \label{eq:regret-bound}
\end{equation}
where $\emph{Regret}^{\emph{std}}_K$ is the cumulative regret of the slow-only agent, and the gap grows with both the total number of interaction steps and the number of update rounds.
\end{theorem}

\begin{explanation}
If the agent maintains and revises a self-model $i_t$ at each step (fast updates) in addition to periodic retraining (slow updates), then it accumulates strictly less regret than an agent that relies on slow updates alone. The advantage comes from better-informed decisions within each round and from higher-quality training data flowing into the next round's slow update.
\end{explanation}

\begin{proof}[Proof Sketch]
The per-step value difference $\Delta_t := V^{g}_{\pi_{k, i_t}, f}(\hat{s}_t) - V^{g}_{\pi_{k, i_0}, f}(\hat{s}_t)$ has strictly positive expectation $\bar{\varepsilon} > 0$ under A1, because the identity evolver succeeds with probability greater than $1/2$ and the bounded degradation on failure is outweighed by the gain on success. Summing over all steps gives a within-round regret reduction of $\sum_k N_k \bar{\varepsilon}$. For the cross-round term, A1 implies that the identity-revised policy collects higher-quality experience, and A2's monotonicity then guarantees that the slow update produces a base policy that is at least as strong as the one the slow-only agent would obtain, yielding a non-negative cross-round advantage $\eta_k \ge 0$ at each round. Combining both terms gives the bound. The formal proof, including the precise probabilistic conditions and the derivation of $\bar{\varepsilon}$, is in Appendix~\ref{appendix:fast-slow-dominates}.
\end{proof}

\begin{figure}
    \centering
    \includegraphics[width=0.55\linewidth]{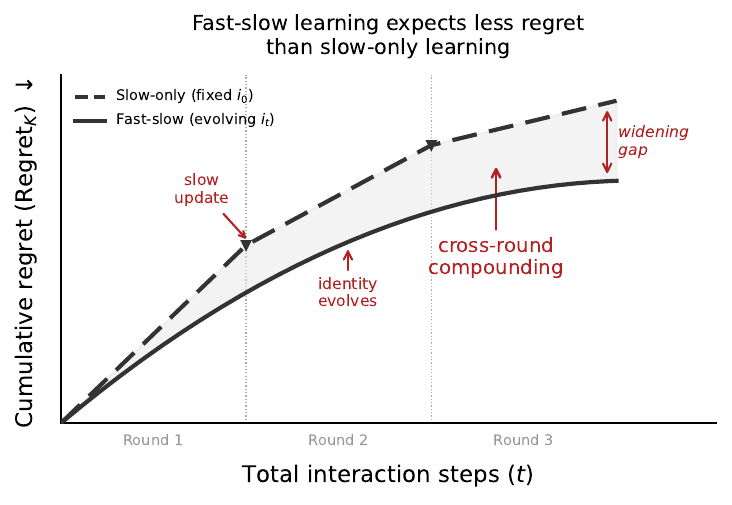}
    \caption{\small An agent that revises its self-model $i_t$ at each step 
    (fast-slow, solid) expects to accumulate less regret than one with 
    fixed identity $i_0$ (slow-only, dashed), as per 
    Theorem~\ref{thm:fast-slow-dominates}. The slow-only curve grows 
    linearly within each round, with slope drops only at round boundaries when slow-update happens
    ($\blacktriangledown$); the fast-slow curve is concave within each 
    round as identity evolution continuously reduces per-step regret. }
    \label{fig:fast-slow-regret}
\end{figure}

Theorem~\ref{thm:fast-slow-dominates} establishes that the fast-slow agent dominates structurally: it optimizes over a strictly larger space $(\theta, i)$ than the slow-only agent $(\theta, i_0)$. The within-round gain is available immediately and requires no further training. The cross-round compounding is realized when slow updates resume and benefit from the higher-quality experience that identity-revised interaction produces (Figure~\ref{fig:fast-slow-regret}).

A natural question following is how identity originates. Unlike the world model, which learns from data the environment supplies, the self-model describes properties of the agent itself, and evidence about them arises only from the agent's own behavior. Identity-bearing corpora (e.g., role descriptions, capability assessments, performance evaluations) teach the vocabulary of self-description but usually describe agents other than the one being trained, while self-model emergent in the agent's own state-action trajectories supply grounded content only for the environments and policy that generated them (\S\ref{sec:am-data}). Both sources therefore yield priors for the initial identity $i_0$, not a finished self-model. A genuine identity emerges only by grounding in the agent's own interaction, with the evolver $\iota$ revising $i_t$ so that what the agent believes about itself answers to realized performance rather than to its initial description. 

One practical benefit of this setup is fast adaptation to new environments or action spaces: 
during deployment, the agent starts from the seeded identity $i_0$ and rapidly adapts its self-model through interaction,
rather than waiting for a human to tune its system prompt.
Identity evolution thereby provides a form of continual learning at test time: the agent keeps learning while it operates, instead of alternating between frozen deployment and scheduled retraining (\S\ref{subsec:critique-learning}).
Like goal decomposition (\S\ref{subsec:critique-goal}), identity adaptation benefits from simulating the hypothetical outcome after assuming a certain identity (e.g., if one sees themself as an experienced negotiator, will they speak more confidently and win a better deal?), 
which draws on the agent's ability for internal simulation (i.e., world model). 
These considerations point toward an architecture in which identity 
serves
as the fast-adapting variable: its revisions should feed immediately into the agent's other decision-making processes (e.g., goal decomposition, planning, and self-regulation), while slower parameter updates consolidate what has proven durable across many such fast revisions.
In practice, the act of identity update can itself be a decision for the agent, as we discuss in detail in \S\ref{subsec:critique-self-regulation}.

\subsection{Decision-Making: From Black-Box Policies to Simulative Reasoning}
\label{subsec:critique-decision-making}

\begin{displayquote}
\textit{Train a sufficiently powerful black-box policy through end-to-end RL; planning capabilities will emerge in the chain-of-thought 
-- does not ground planning in real-world dynamics.
}
\end{displayquote}

A dominant instinct in current agent design is to treat the system as a single black-box policy: given the current observation $o_t$, the policy generates a sequence of intermediate latent variables $z_t$ (e.g., hidden-layer activations~\cite{hinton1995wake_sleep,dehghani2018universal} or chain-of-thought tokens~\cite{wei2022chain}) before emitting the next action. 
The hypothesis is that scaling this architecture and training it with massive demonstration data and/or reinforcement learning will cause advanced capabilities such as ``planning'' to emerge inside the intermediate representations, as has been recently advocated by Florence from Generalist AI~\cite{florence2026beyond}. 
This view is attractive because it is simple, aligns with the recent success of scaling next-token prediction~\cite{brown2020language} and chain-of-thought reasoning~\cite{guo2025deepseek}, and offers a clean training story: learn one powerful reasoning policy, and let it handle everything. 

We argue that this view conflates two distinct concepts: \textbf{internal compute} and \textbf{planning}. A neural network can learn to compute precise hidden-layer activations or generate useful reasoning tokens, ultimately better fitting its training data. This by itself, however, does not provide the core primitive that planning requires: a grounded way to reason about counterfactual environment dynamics (i.e., what would happen if we took action $a$ from state $s$), 
due to the lack of structure and supervision to that effect.
Indeed, agentic reasoning is fundamentally a control problem: estimating the world state $\hat{s}$, proposing candidate actions $\{a\}$, predicting their outcomes $\{\hat{s}'\}$, estimating goal progress $\{V\}$, and selecting the best action $a^*$ while accounting for prediction reliability. Current reasoning models (e.g., o1~\cite{openai2024o1}, R1~\cite{guo2025deepseek}) generate extended chains of thought that may \emph{describe} possible futures, but these descriptions are not grounded in a model that predicts state transitions from observations. The result is prediction based on narrative plausibility (e.g., token probability) rather than real-world consistency, with no guarantee of correct planning. 
As Xing et al.~\cite{xing2025critiques} argue, text can be a powerful component of world-state representation, but only when anchored to real-world dynamics through a world model trained with objectives grounded in data reconstruction. 
Without such grounding, more reasoning tokens can simply mean more opportunities for confident but unfounded extrapolation. 
A \textbf{world model}, which takes the current estimated state $\hat{s}$ and action, and predicts the next state $\hat{s}'$, thus emerges as the missing component that enables grounded decision-making based on predicted outcomes, detecting when the system is extrapolating beyond its competence and improving planning reliably without entangling it with the entire policy. 

\begin{figure}
    \centering
    \includegraphics[width=\linewidth,page=1]{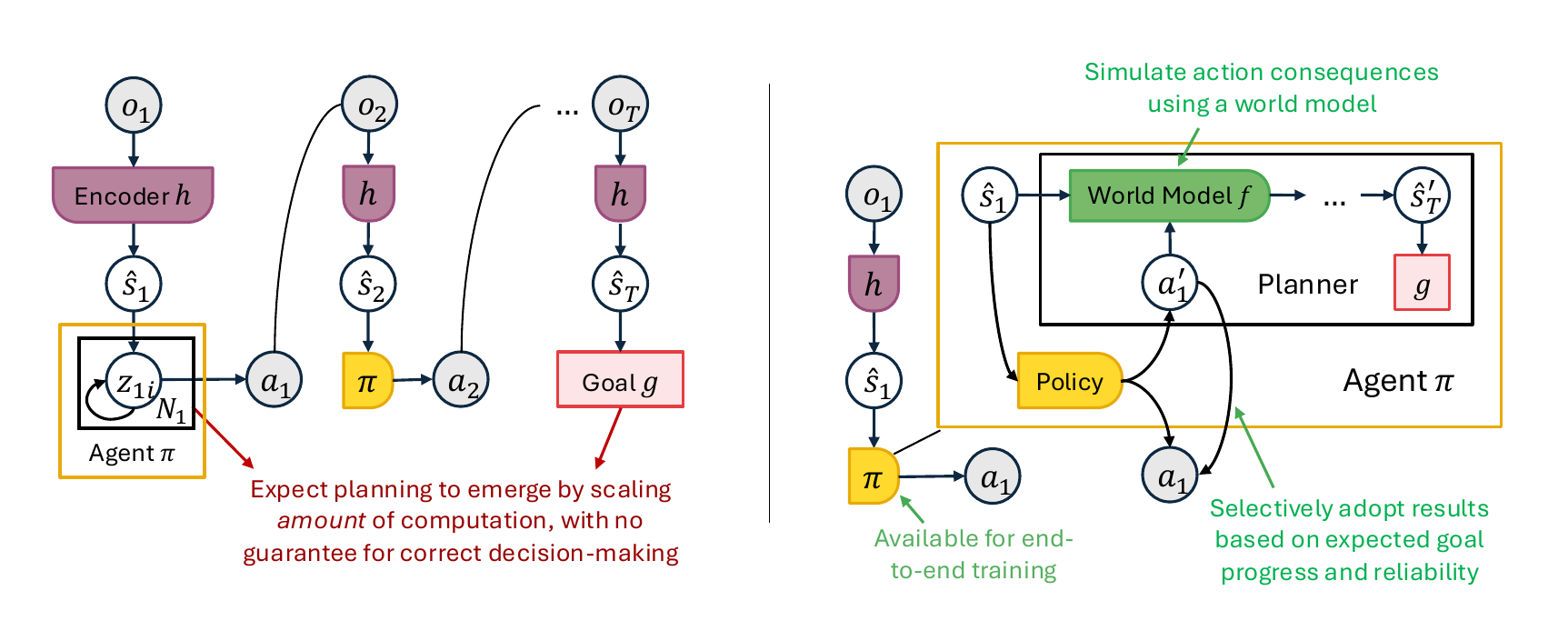}
    \caption{\small \textbf{Comparison of reactive policy (System~I) and simulative reasoning (System~II).} (\textbf{Left}) A reactive policy maps observations to actions through unconstrained intermediate variables (e.g., hidden activations or chain-of-thought tokens). Reasoning is based on narrative plausibility rather than grounded dynamics, without guarantee of correct decision-making. (\textbf{Right}) Simulative reasoning uses a world model $f$ to predict the consequences of candidate actions, evaluating goal progress through a critic $v$, and selecting the best action while accounting for prediction reliability.
    The critic module is not depicted.
    }
    \label{fig:decision-making}
\end{figure}

Our position is therefore not that reactive policies cannot reason, nor that agents should always plan. Rather, even with a strong baseline policy $\pi$, introducing an explicit world-model-based simulation component $f$, when used selectively based on its reliability, provides the missing counterfactual engine. This claim can be made precise: as we show formally in Theorem~\ref{thm:world-model-improve-policy}, 
if a reasonably accurate world model exists, \emph{any} baseline policy can be augmented with it to obtain a mixed policy $\pi_{\text{mix}}$ that is at least as good, if not better.
\begin{theorem}[World-Model-Based Planning Improves Any Policy]
    \label{thm:world-model-improve-policy}
    Given a world model $f$ such that given any state-action pair $(s, a)$, relative to the universe $\mu$, the prediction error for the next state $s'$ is bounded in terms of total variation (TV) as below:
    $$\text{TV}\left( p_f(s' \mid s, a), p_\mu(s' \mid s, a) \right) \leq \epsilon.$$
    Also assume discount schedule $\{\gamma_k\}_{k=t}^\infty$ where $\gamma_k = \gamma^{k-t}$ for $\gamma \in (0, 1)$, and the reward is bounded as $r(g, s) \leq R_\text{max}$. Then for any policy $\pi$, there exists $\pi_\text{mix} = \phi(f,\pi, \epsilon)$ such that $$V^g_{\pi_\text{mix}} \geq V^g_{\pi}.$$
\end{theorem}
\begin{explanation}
    If you have a reasonably accurate world model $f$, then you can augment any baseline policy $\pi$ with it to obtain a mixed policy $\pi_{\text{mix}}$ which will perform better than, or at least equal to, the original policy.
\end{explanation}
\begin{proof}[Proof Sketch]
First, we observe that based on the Simulation Lemma~\cite{kearns2002near}, if the world model $f$ approximates the true environment $\mu$ closely, then the state values and Q-values they produce will differ at most by a small error $\frac{2 \gamma R_{\text{max}} \epsilon}{(1-\gamma)^2} \vcentcolon= \epsilon_{\text{model}}$.
Next, given any policy $\pi$, we define a mixed policy $\pi_{\text{mix}}$ that follows the best action selected by world-model-based planning $\pi^*_f$ only when its value is more than $2 \epsilon_{\text{model}}$ higher than that of $\pi$. Because of this margin, whenever $\pi_{\text{mix}}$ follows $\pi^*_f$, it would be a true improvement on $\pi$ in the real environment. Otherwise, it just falls back to $\pi$. Finally, the Performance Difference Lemma~\cite{kakade2002approximately} shows this guarantees $\pi_{\text{mix}}$ achieves at least the same value as $\pi$, and strictly better whenever the WM's improvement is adopted at least once.
\end{proof}
The detailed proof can be found in Appendix~\ref{appendix:world-model-improve-policy}. Note that uniform improvement calls for selective planning: the mixed policy follows the world-model-based plan only when its predicted improvement exceeds a safety margin for model error, and falls back to the baseline otherwise. Even a strong policy is never made worse, and is strictly improved whenever the world model identifies a better action.
Note also that the theorem's premise of a TV-bounded prediction error $\epsilon$ is only credible when the world model is trained for predictive fidelity. If the world model's parameters were instead shaped by the agent's reward objective, $\epsilon$ would no longer measure distance from reality, and the guarantee would be vacuous; we return to this point in \S\ref{subsec:critique-learning}.

We call this form of decision-making \textbf{simulative reasoning} (Equation~\ref{eq:world-model-decision-making}), which intuitively corresponds to \textbf{System~II}, the part of human deliberation that is slow but accurate and precise, as discussed by Kahneman~\cite{kahneman2011thinking}. This is distinguished from the original \textbf{reactive policy}, which can be described as \textbf{System~I}, the decision-making process that is fast but prone to biases and errors.

In simulative reasoning, the agent proposes candidate actions, predicts their consequences through the world model, evaluates goal progress, and selects the best action, performing thought experiments computationally with controllable depth and breadth. 
Note that this process need not be programmed using traditional search algorithms (e.g., DFS, MCTS), but can be absorbed by the inference procedure of an end-to-end system
in which the policy, world model, and other modules exchange activations under structured attention patterns (\S\ref{subsec:the-gic-arch}), while each remains trained under its own objective.
Plans generated through this process $c_t$ (Equation~\ref{eq:simulative-planning}) can be maintained in an associative memory, reducing redundant computation and preserving continuity of intent across steps. 
In practice, it is also possible to distill the results from System~II into System~I, opening up a credible path to training a stronger reactive policy when latency is a concern. 
The question of \emph{when} to invoke simulative reasoning vs.\ acting directly is itself a decision that should be governed by the agent, which we discuss next. 

\subsection{Self-Regulation: From Fixed Workflows to Learned Configurators}
\label{subsec:critique-self-regulation}


\begin{displayquote}
\textit{Either expect effective deliberation to emerge from unconstrained RL, 
or prescribe it through fixed workflow stages 
-- neither lets the agent regulate its own reasoning.}
\end{displayquote}

Given that both reactive action (System~I) and simulative reasoning (System~II) are available, a second question arises as to \emph{how} to decide which decision mode to engage. Different situations call for different amounts and types of internal computation, depending on urgency, difficulty, uncertainty, and resource budget. Current practice address this question in one of two ways, neither of which is satisfactory. 

The first approach is to expect effective deliberation patterns to emerge from unconstrained chain-of-thought during RL training (e.g., DeepSeek-R1). Within this paradigm, however, there is no explicit control for when the model will perform slow, deliberate planning vs. fast, instinctive reacting, nor bound over inference-time compute or reasoning budget. As a result, reasoning compute can increase dramatically during training, while longer reasoning does not necessarily yield better answers~\cite{gema2025inverse,su2025underthinking}. Effort to control reasoning cost has resulted in ``adaptive thinking models'' (e.g., GPT-5~\cite{openai2025gpt5}, Opus-4.7~\cite{anthropic2026claudeopus47}) which receive mixed reviews from end users~\cite{newton2025gpt5Backlash,digitaltoday2026opus47Backlash}.  

The second approach is to build planning into a fixed, externally prescribed stage of the workflow. Examples include human-controlled planning-execution pipelines (e.g., plan mode in Claude Code), scripted reasoning loops (e.g., CRISPR-GPT~\cite{qu2025crispr}), and always-on model-predictive control (MPC as advocated by LeCun~\cite{lecun2022path}). While more structured and amenable to customization and injection of domain expertise, these approaches introduce their own limitations. 
Fixed planning stages and reasoning pipelines force expensive deliberation even when direct action suffices. 
MPC, in particular, must replan from scratch at each step, losing continuity of intent and incurring high computational overhead.
Moreover, MPC's fixed planning horizon is fundamentally limited: as we show formally in Theorem~\ref{thm:horizon-requirement-mpc}, the required simulation horizon $H$ grows significantly with higher desired planning precision.

\begin{theorem}[Horizon Requirements for Pure $H$-step MPC in the World Model]
\label{thm:horizon-requirement-mpc}
Let $f$ be the world model with transition kernel $p_f(s' \mid s, a)$, let $\pi^*$ denote the optimal policy acting in $f$, namely $\pi^* \coloneq \argmax_{\pi} V^g_{\pi,f}$,
and let $C_g: \mathcal{S} \to [0, C_\text{max}]$ be a cost function. Given planning horizon $H \geq 1$ and assuming the discount schedule $\gamma_k = \gamma^{k-t}$ for $\gamma \in (0, 1)$, consider a $H$-step MPC policy which, given state $s_t$, simulates up to time step $T = t + H$ for decision-making as below:
\begin{equation}
    \pi^H_{\text{MPC}}(s_t) = \argmin_{a_t, \dots, a_{T-1}} \sum_{s_{t+1}:s_{T}} \left[ \sum_{k=t}^T \gamma^{k-t} C_g(s_k)  \right] \prod_{i=t}^{T-1} p_f(s_{i+1} \mid s_i, a_i).
\end{equation}
Assume the cost function is perfectly aligned with the original goal reward, meaning there exists a goal-dependent constant $b_g$ such that $C_g(s) = b_g - r(s, g)$. Then, given $\epsilon > 0$, to achieve $\lVert V^g_{\pi^*,f} - V^g_{\pi^H_\text{MPC}, f} \rVert \leq \epsilon$, it suffices that: 
$$
H = O\left(\frac{1}{1 - \gamma} \left[ \log\frac{1}{\epsilon} + 2 \log\frac{1}{1 - \gamma} + \log C_{\text{max}} \right] \right).
$$
If $\gamma$ and $C_{\text{max}}$ are treated as constants, then:
$$
H = O\left(\log \frac{1}{\epsilon} \right).
$$
\end{theorem}

\begin{explanation}
    Pure MPC can reduce planning error by increasing the lookahead horizon, but the required simulation depth increases significantly with precision demands; the cost becomes increasingly demanding for precise planning, let alone running it for every decision with a fixed planning horizon $H$.
\end{explanation}

\begin{proof}[Proof Sketch]
Because the cost function is perfectly aligned with reward, minimizing cost is equivalent to maximizing the shifted reward $\tilde{r}(s, g) = -C_g(s) = r(s, g) - b_g$, which does not change the optimal policy or value gap we want to bound. 
Let $\tilde{T}$ be the Bellman operator under $\tilde{r}$, where applying $\tilde{T}$ once means looking one step ahead and then using a continuation value. 
Pure $H$-step MPC policy $\pi^H_{\text{MPC}}$ can then be viewed as acting greedily with respect to the finite-horizon estimate $\hat{V}^{(H-1)} = \tilde{T}^{H-1} 0$, namely rolling out for $H$ steps and assigns zero value to the unplanned future. 
By standard approximate-greedy bound, its suboptimality is controlled by $\lVert \tilde{V}^* - \hat{V}^{(H-1)} \rVert_\infty$. Bellman contraction gives $\lVert \tilde{V}^* - \tilde{T}^{H-1} 0 \rVert_\infty \leq \gamma^{H-1} \lVert \tilde{V}^* \rVert_\infty$, and bounded cost implies $\lVert \tilde{V}^* \rVert_\infty \leq C_{\text{max}} / (1 - \gamma)$. 
Combining these yields $\lVert V^g_{\pi^*,f} - V^g_{\pi^H_{\text{MPC}}} \rVert \leq 2 \gamma^H C_{\text{max}} / (1 - \gamma)^2$, so achieving error at most $\epsilon$ requires $H$ large enough that the derived bound is below $\epsilon$.
\end{proof}

\begin{figure}
    \centering
    \includegraphics[width=0.5\linewidth]{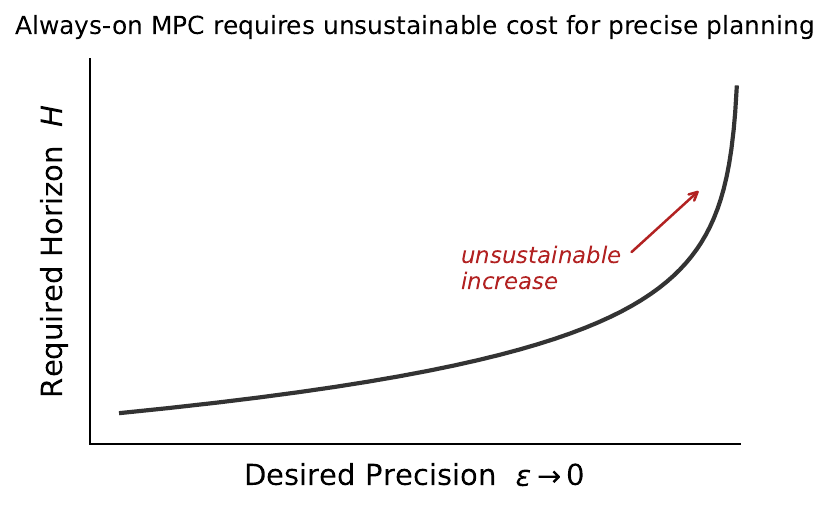}
    \caption{\small As the desired planning precision increases ($\epsilon \to 0$ as per Theorem~\ref{thm:horizon-requirement-mpc}), the required planning horizon $H$ grows significantly. For an always-on, fixed-depth MPC routine, this means that any choice of horizon is either too shallow to achieve the target precision or too deep to be computationally feasible at every timestep. This motivates moving beyond always-on planning toward approaches that allow the agent to decide for itself when and how deeply to deliberate. }
    \label{fig:horizon-vs-precision}
\end{figure}

As Theorem~\ref{thm:horizon-requirement-mpc} and Figure~\ref{fig:horizon-vs-precision} show, increasing the desired planning precision ($\epsilon \to 0$) results in increasing demands on the planning horizon $H$. In particular, always-on, fixed-depth MPC commits to a uniform planning procedure at every decision point, which results in overplanning in easy states where simple reactive policy suffices, and underplanning in difficult or high-stakes states that require deep and detailed simulation. 
Fundamentally, neither scripted pipeline nor fixed MPC allows the agent to decide \emph{for itself} when and how deeply to deliberate, effectively externalizing another dimension of agency that should have been internal to the agent.

\begin{figure}[t]
    \centering
    \includegraphics[width=\linewidth,page=2]{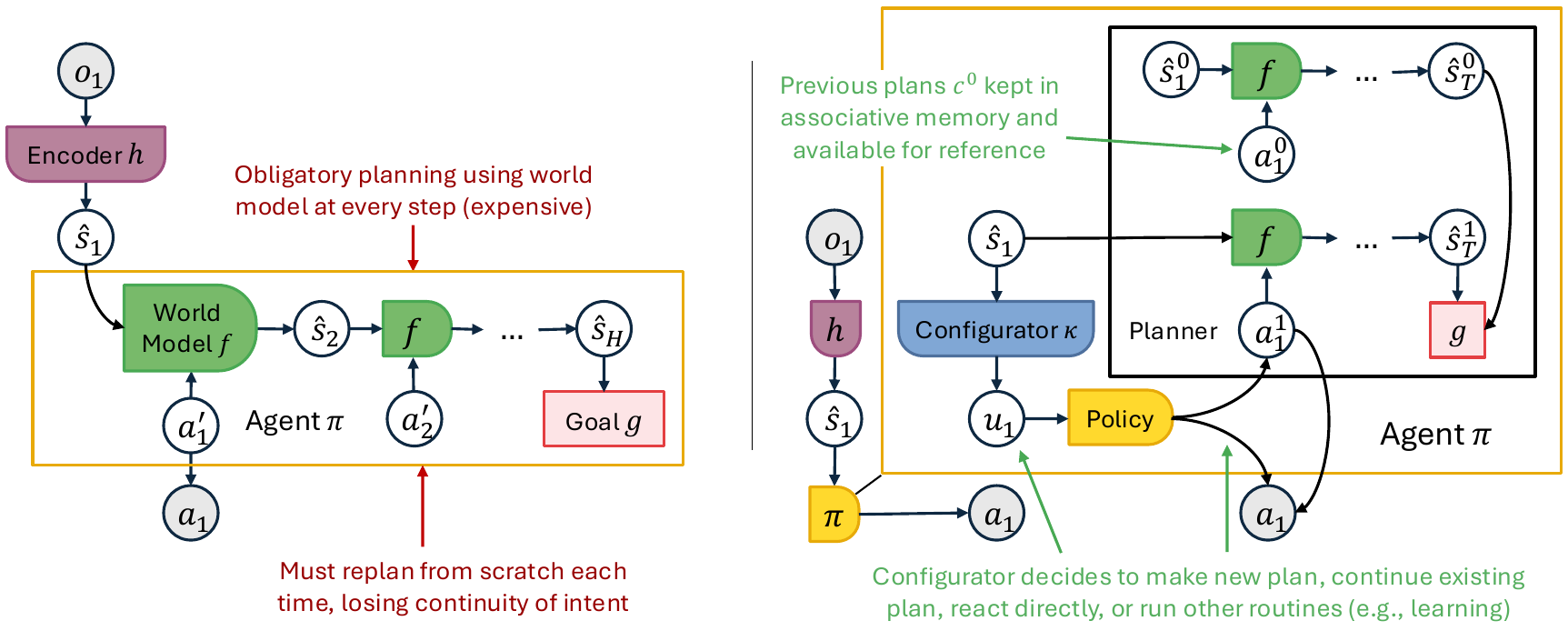}
    \caption{\small \textbf{Comparison of model-predictive control (MPC) and self-regulated simulative reasoning (System~III + System~II).} (\textbf{Left}) MPC applies a fixed-depth planning tree of horizon $H$ at every decision step, regardless of situation difficulty. Plans are discarded and rebuilt from scratch at each step, resulting in overplanning during routine situations and underplanning during critical ones. (\textbf{Right}) A learned configurator $\kappa$ decides whether to make new plan via simulative reasoning (System~II), continue an existing plan, react directly (System~I), or run other routines (e.g., learning). Previous plans are cached in associative memory and available for reference.
    This allows the agent to invest computation where it matters while avoiding the uniform overhead and discontinuous intent of always-on planning.
    }
    \label{fig:self-regulation}
\end{figure}

The constructive alternative is a learned \textbf{configurator} $\kappa$, formalized in \S\ref{sec:boundary} and illustrated in Figure~\ref{fig:self-regulation}, which outputs a regulation decision $u_t$ at each step that governs the agent's deliberative mode: construct a new simulative plan, continue or revise an existing one, or skip planning entirely and act directly. Both Systems~I and II are needed for human-level agency; what matters is that the agent itself selects the appropriate mode based on urgency, difficulty, uncertainty, and resource budget. As the configurator models the meta-cognition that dynamically switches between these two systems, we analogously refer to this process as \textbf{System~III}. 
The configurator itself should be trained (e.g., via RL) as \underline{part of the agent's policy} to maximize task success while managing computational expenditure, and can adapt its regulation strategy with experience. As such, the meta-decision-making may also be enhanced by simulative reasoning using the world model $p_f(s_{t+1} \mid s_t, u_t)$ which predicts the abstracted consequence after adopting a specific deliberation mode.

In practice, the regulation variable $u_t$ may encode more nuanced decisions, such as choosing not to pursue certain subgoals or take certain actions. Indeed, from a safety perspective, certain behaviors considered objectionable in general may be critical to safety in other scenarios (discussed in more detail in \S\ref{sec:am-safety}). For instance, crossing a room at a leisurely pace vs. sprinting to retrieve an epipen for someone with a life-threatening allergic reaction involves the same motor system but entirely different configurations: in the latter, knocking objects aside becomes acceptable, social norms about running indoors are suspended, and physical cost to oneself is discounted. 
Self-regulation, in this view, functions not merely as a computation scheduler, but also like human emotion: configuring behavioral modes that structure the agent's priorities and action repertoire based on situational assessment. 
The configurator also plays a role in deciding when and how the agent should learn from experience, as we discuss next. 

\subsection{Learning: From Human-Designed Pipelines to Self-Directed, Simulative Improvement}
\label{subsec:critique-learning}


\begin{displayquote}
\textit{Train the agent through human-designed pipelines (e.g., RL 
in simulators, supervised demonstration), and deploy a frozen checkpoint
-- does not allow the agent to govern its own learning.}
\end{displayquote}

Current approaches to training agents cluster around three main positions. 
The first trains the policy via RL in rule-based simulators or ``digital twins'' for cheap scalability, easy reversibility, and safe trial-and-error. Examples include code-based 3D simulators from MoonLake AI (supported by Manning and Goodfellow~\cite{manning2026efficient}) and exported assets from 3D-scene models (e.g., World Labs supported by Fei-Fei Li~\cite{worldlabs2025marble}). 
The second trains in the real environment with supervised correction, arguing that no simulator yet matches reality, a position championed by Levine~\cite{levine2025sporks}. 
The third, advocated most prominently by LeCun~\cite{lecun2022path}, argues that training a world model (WM) via self-supervision is sufficient, and that learning a separate policy through RL is inefficient and unnecessary. 
Each of these positions captures an important aspect of the training problem.
However, they share a common structural property: in all three cases, training is treated as a finite phase, scheduled, curated, launched, and monitored by human engineers, that terminates before deployment. 
We argue below that this shared assumption leaves significant room for a more complete treatment of agency.


\paragraph{Program as Simulator vs. Model as Simulator.} Rule-based simulators (e.g., MoonLake AI and World Labs) have demonstrated impressive results within their target domains, but as computer programs, they are inevitably bounded by the scope of 3D engineering and the ability to analytically model every nuance of the real world.
An AI-driven WM (e.g., JEPA~\cite{assran2025vjepa2selfsupervisedvideo} from AMI and GLP~\cite{xiang2025pan} from IFM), however, is fundamentally different from a hand-crafted digital twin or a metaverse, due to its use as a simulator built through data-driven machine learning. 
Given appropriate architecture and sufficient data, a learned simulator can converge towards accurate simulation of real-world dynamics in a way no hand-engineered program can match in general. 
The distinction is analogous to the shift from hand-crafted features to learned representations in computer vision (e.g., AlexNet~\cite{krizhevsky2012imagenet}) -- what changed was not the problem, but the recognition that \emph{learning} scales where engineering does not. 

\paragraph{Simulation-First, Reality as Validation.}
An influential perspective (e.g., as articulated by Levine) treats reality as the primary training arena and simulation as a supplement.
But for many domains (e.g., climate intervention, drug discovery, aerospace missions, military conflicts), real-world trial-and-error is dangerous, expensive, or irreversible. 
Just as one would not put a pilot in a real plane on their first day, the machine should follow the inverted principle: \textbf{simulate first, use reality as validation}. Specifically, the agent should learn primarily from its world model as a simulator, and then use real interaction to validate and calibrate the simulator, not as the default learning environment. 
This principle is not merely an engineering convenience, but also has formal support. As we prove formally in Theorem~\ref{thm:mixture-beats-real-only}, given a fixed budget of real experience, augmenting it with world-model-simulated experience yields policies 
with a good chance of outperforming the real-only policy, even if the WM is not perfect.
When the world model is perfect, the mixture dominates with certainty. 
\begin{theorem}[Mixture of simulated and real experience outperforms real-only experience for training agents, up to world-modeling error terms]
    \label{thm:mixture-beats-real-only}
    Given a fixed dataset of real experience collected from the true environment $\mu$: $D_{\mu} = \{ (s, a, s', r') \}_{i=1}^{N_\mu}$,
    define two hypothesis sets of policies computable from the interaction budget $D_\mu$:
    \begin{itemize}
        \item $\Pi_{\text{env}}(D_\mu)$: All policies that can be computed using only $D_\mu$, namely experience from the real environment.\footnote{Note that no limitation is placed on the nature of the experience nor the learning method: $D_\mu$ may be either an offline demonstration dataset or an experience buffer collected through on-policy exploration, and the policy may consume the experience through either imitation learning or other reinforcement learning algorithms.}
        \item $\Pi_{\text{mix}}(D_\mu, D_f)$: All policies that can be computed using a mixture $M_\alpha = (1 - \alpha) \mu + \alpha f$ of the real experience $D_\mu$ and simulated rollouts $D_f = \{(s, a, s', r')\}_{i=1}^{N_f}$ from the world model $f$.
    \end{itemize}
    Further define the best-possible policy given only real experience $\pi^*_\text{env}$ and given the mixture experience $\pi^*_\text{mix}$, respectively, as below:
    \[
    \pi^*_\text{env} = \argmax_{\pi \in \Pi_{\text{env}}(D_\mu)} V^g_{\pi,\mu}, \qquad \pi^*_\text{mix} = \argmax_{\pi \in \Pi_{\text{mix}}(D_\mu, D_f)} V^g_{\pi,M_\alpha}.
    \]
    Then, the following inequality holds:
    $$
    V^g_{\pi^*_\text{mix}, \mu} \geq V^g_{\pi^*_{\text{env}}, \mu} - 2 C(\gamma, R_\text{max}) \alpha \epsilon,
    $$
    with $V^g_{\pi^*_\text{mix}, \mu} \geq V^g_{\pi^*_{\text{env}}, \mu}$ when the world model $f$ is perfect ($\epsilon_f = 0$).
\end{theorem}
\begin{explanation}
    If the agent has access to both real experience and simulated experience from a world model, then the best policy it can learn 
    has a good chance of outperforming the best policy learned from real experience alone, with the chance tied to the world model's accuracy. With a perfect world model, the mixture dominates with certainty.
\end{explanation}
\begin{proof}[Proof Sketch]
First, the mixed-experience policy class contains the real-only policy class (i.e., $\Pi_{\text{env}}(D_\mu) \subseteq \Pi_{\text{env}}(D_\mu, D_f)$), since a learner with access to both real and simulated experience can always ignore the simulated data. Therefore, the best mixture-trained policy $\pi^*_{\text{mix}}$ must achieve at least as much value as the best real-only policy $\pi^*_{\text{env}}$, when both are evaluated in the mixed environment $M_\alpha$. 
Second, by the Simulation Lemma, evaluating any fixed policy in $M_\alpha$ instead of the true environment $\mu$ introduces at most $C(\gamma, R_\text{max}) \alpha \epsilon$ value error. Applying this error bound once to transfer $\pi^*_\text{mix}$'s value from $M_\alpha$ back to $\mu$ and once to transfer $\pi^*_\text{env}$'s value from $\mu$ to $M_\alpha$, giving a total penalty of $2C(\gamma, R_\text{max}) \alpha \epsilon$. When the world model is perfect, the simulation error $\epsilon$ is zero, resulting in domination by $\pi^*_\text{mix}$.
\end{proof}
The detailed proof can be found in Appendix~\ref{appendix:mixture-beats-real-only}.
In contrast with mixed-experience training, real-world-only training, while grounding the agent in true dynamics, is insufficient for tasks that are unsafe, expensive, or slow to provide feedback.
In particular, PAN~\cite{xiang2025pan} emerges as an example of a WM that can support general simulative learning as discussed above. Built on the generative latent prediction (GLP) architecture, PAN is trained to support open-domain, action-conditioned simulation with coherent, long-term dynamics. One particular advantage of PAN compared to latent-only WMs (e.g., V-JEPA 2~\cite{assran2025vjepa2selfsupervisedvideo}) is its ability to decode simulation back to observation space (e.g., videos) for collaboration with a wide range of downstream systems (e.g., vision-language, robotic, and autonomous-driving models), as recently argued in the debate on world models between Xing and LeCun~\cite{lecun_xing_world_models_2026}.

\paragraph{Learning to Predict vs. Learning to Act.}
Training a WM through self-supervision is necessary but, as we argue, not by itself sufficient. 
Self-supervised learning (SSL) produces a WM capable of next-state prediction, which is valuable as a substrate for simulative reasoning (\S\ref{subsec:critique-decision-making}) and provides a learned simulator for generating training experience (Theorem~\ref{thm:mixture-beats-real-only}).
However, the WM predicts what \emph{will happen}; the AM decides what \emph{to do}. 
No amount of SSL produces an agent that decomposes goals, evolves identity, configures decision modes, and selects actions to maximize long-term goal success, any more than a perfect flight simulator produces a trained pilot. 
As discussed in \S\ref{subsec:critique-self-regulation}, relying on MPC to bridge the prediction--action gap faces fundamental horizon limitations (Theorem~\ref{thm:horizon-requirement-mpc}). RL thus remains essential not as a refinement step on top of SSL, but as the paradigm that trains the AM to act effectively \emph{within} and \emph{through} the WM, never \emph{as} the WM.

This can be seen as an instance of the broader conflation of world model and agent model discussed in \S\ref{subsec:agent-environment-model}.
Recent work~\cite{ye2026dreamzero,li2026functional_taxonomy_world_models,nvidia2026cosmos3} labels action generation as part of the WM’s capability and trains joint world-action architectures. 
Such integration is a legitimate engineering choice for end-to-end optimizability, but can obscure a functional distinction between WMs trained for next-state prediction and AMs trained for reward maximization.
When the WM's predictions are supervised by a reward-maximizing objective, the model is biased towards optimistic states that, without complex heuristics (e.g., realism penalties, advantage weighting, hyperparameter selection), can be easily exploited by the policy for degraded performance in practice, an insight well-documented in model-based RL~\cite{eysenbach2022mismatched,mete2026optimistic}.
The separation we advocate therefore operates at three levels: \emph{function} (next-state prediction vs. action selection) and \emph{training objective} (prediction loss vs. reward) must always be kept distinct, while \emph{architecture} remains free to integrate the two models end-to-end, as we show in \S\ref{subsec:the-gic-arch}.

\paragraph{External Learning Schedule vs. Internally Regulated Learning}
In current approaches~\citep[e.g.,][]{slime_github,cadene2024lerobot}, when to learn, what data to use, and when to stop are decisions made by human engineers, not by the agent. This not only exogenizes a core aspect of genuine agency, but also risks replacing the long-term potential of goal-oriented learning with the short-term convenience of manual engineering. 
The constructive alternative treats learning as perpetual and self-directed. The agent should govern its own learning process, deciding when to execute in the environment, when to retreat into simulation for practice, when to update the world model from recent experience, and when to revise its self-model. 
In the fully realized vision, \textbf{perpetual learning} takes two complementary forms. The first is learning through real interaction: working on problems changes the agent's internal decision-making structure, not just produces outputs. This is fundamentally different from typical ``reflection'' mechanisms that generate self-evaluative text but leaves the agent's parameters untouched~\cite{shinn2023reflexion}. 
The second is learning through imagined experience: when not actively engaged in the real world, the agent uses its world model to generate hypothetical scenarios and learns from them (i.e., RL from a simulated world), requiring no real-world interaction at all. 
An agent that interleaves execution and self-improvement in this way is qualitatively different from one that is frozen after deployment.

\subsection{Summary: Agent Model \emph{with} World Model}
\label{subsec:critique-summary}

The common thread across the critique above is that current systems externalize the structures of agency (i.e.,
goals, identity, decision-making, self-regulation, and learning) into human-engineered scaffolding. 
A truly agentive system possessing endogenous artificial agency requires that
each dimension in question points toward the same constructive alternative: \emph{internalizing} these structures within a unified learned model. 

Furthermore, every constructive alternative, as has emerged from the discussion, relies on or benefits from the agent's ability to simulate reality internally. 
Goal decomposition requires predicting consequences to assess the feasibility and ordering of subgoals. Identity evolution requires simulating one's own performance to revise self-assessment. Decision-making requires predicting state transitions to ground counterfactual reasoning. Self-regulation requires assessing situational difficulty and urgency to select the appropriate behavioral mode. And learning requires a learned simulator to generate experience faithfully, safely, and at scale. 

The \textbf{world model} thus emerges not as one component among many, but as the connective substrate through which the other dimensions of agency become possible.
As argued in \cite{xing2025critiques}, building a general-purpose learned simulator of the world is not merely an engineering component of agent design, but a goal of AI in its own right --- a system that, given the right architecture and sufficient data, can converge toward faithful simulation of diverse real-world dynamics. Agents are the way to extract value from such a simulator: the relationship between the agent and the world model is analogous to that between a pilot and a flight simulator, where the simulator provides the substrate for both reasoning and learning, and the agent provides the intentionality that turns simulation into purposeful action. 


This convergence motivates the architecture we present next: a unified \textbf{agent model} in which goal decomposition, identity evolution, simulative reasoning, self-regulation, and self-directed learning arise as components of a single adaptive system, 
paired with a separately learned world model that the agent consults as its internal simulator in planning and its arena for continuous improvement.

\section{The GIC Agent Model}
\label{sec:gic-agent-model}


The critique in \S\ref{sec:critiques} converges on six design requirements for  
achieving capability akin to that of genuine agency in an agentive artificial system:
\textbf{persistent goals} with hierarchical decomposition;
\textbf{evolving identity} that updates with experience;
\textbf{simulative reasoning} through an internal world model;
\textbf{self-regulation} via a learned configurator; and
\textbf{self-directed learning} from both real and simulated experience.
Meeting these requirements calls for a single learned model that generates distributions over actions conditioned on world state, goals, identity, and plans. This is not merely predicting the next token in a sequence, but simulating the full distribution of possible actions and their consequences, parallel to the world model's simulation of possible worlds~\cite{xing2025critiques}.
We refer to such a model as an \textbf{Agent Model} (AM). 
In this section, we present \textbf{Goal-Identity-Configurator} (GIC), an architecture for agent models, and describe its training, deployment, evaluation, data requirements, and safety properties.
Details and preliminary results for specific, scaled-down instantiations shall appear in companion manuscripts~\citep[e.g.,][]{deng2026generalagenticplanningsimulative,deng2026sr2am}.

\subsection{A Motivating Use Case: Training an Aircraft Pilot}

A truly versatile and autonomous agent model must handle the full complexity of real-world behavior: variations in modality (e.g., verbal, visual, proprioceptive, tactile), temporal scope (e.g., split-second reflexes to multi-day campaigns), action granularity (e.g., fine motor control to strategic decisions), and social structure (e.g., solo operation to coordinated teams).
We therefore ground our discussion in a more demanding use case: the training and deployment of an aircraft pilot, which naturally stages every component of the agent model across a developmental arc.

\paragraph{Ground School}
The process begins with classroom learning (manuals, regulations, meteorology, aerodynamics) that builds an internal world model of flight physics and procedures. 
Extensive browsing of book knowledge (e.g., philosophy, cultural stories) builds the vocabulary for abstract concepts (e.g., ideology, loyalty, values, and morality), while lack of operating experience leads to realistic self-awareness of skill level (e.g., ``I know the rules but have never flown.''). Both of these serve as the basis of future identity development.

\paragraph{Simulator Training}
In the flight simulator, the pilot builds reactive competence (System~I: e.g., stick-and-rudder coordination), deliberate planning (System~II: e.g., fuel management), and the ability to shift fluidly between modes (System~III). Identity in terms of skill awareness evolves (e.g., ``I can land in crosswinds but am weak on instrument approaches.''), while philosophical values are ingrained in response to task curriculum (e.g., learning when to prioritize mission and when to preserve oneself).

\paragraph{Real-Aircraft Deployment}
After simulator comes deployment to a real aircraft, which forces online adaptation to the sim-to-real gap (e.g., G-forces, vibration, fatigue, visual illusions) and goal decomposition (e.g., a cross-country flight into legs, waypoints, and altitude management). 
The pilot's identity in terms of skill odometer and personal values are challenged and calibrated by the real experience (e.g., maintaining composure in face of sudden engine stall).

\paragraph{Fleet Coordination}
Later, the pilot may join a fleet, where communication and coordination arise as task necessities (e.g., leading or following based on each pilot's model of teammates' capabilities) rather than external assignment.
The identity further develops to encompass new relationships and instilled team values.

\paragraph{Command}
At the strategic level, a pilot-turned-commander reasons over multi-day campaigns, logistics, adversaries, and personnel, planning across time scales and deciding which decisions to make personally and which to delegate.
In their leadership capacity, the commander may also play a role in shaping the identities of their subordinates through example, teaching, and organizational structures.

A single cognitive architecture underlies this entire trajectory. The challenge is building a model that supports it.


\subsection{The GIC Architecture}
\label{subsec:the-gic-arch}

\begin{figure}
    \centering
    \includegraphics[width=\linewidth,page=2]{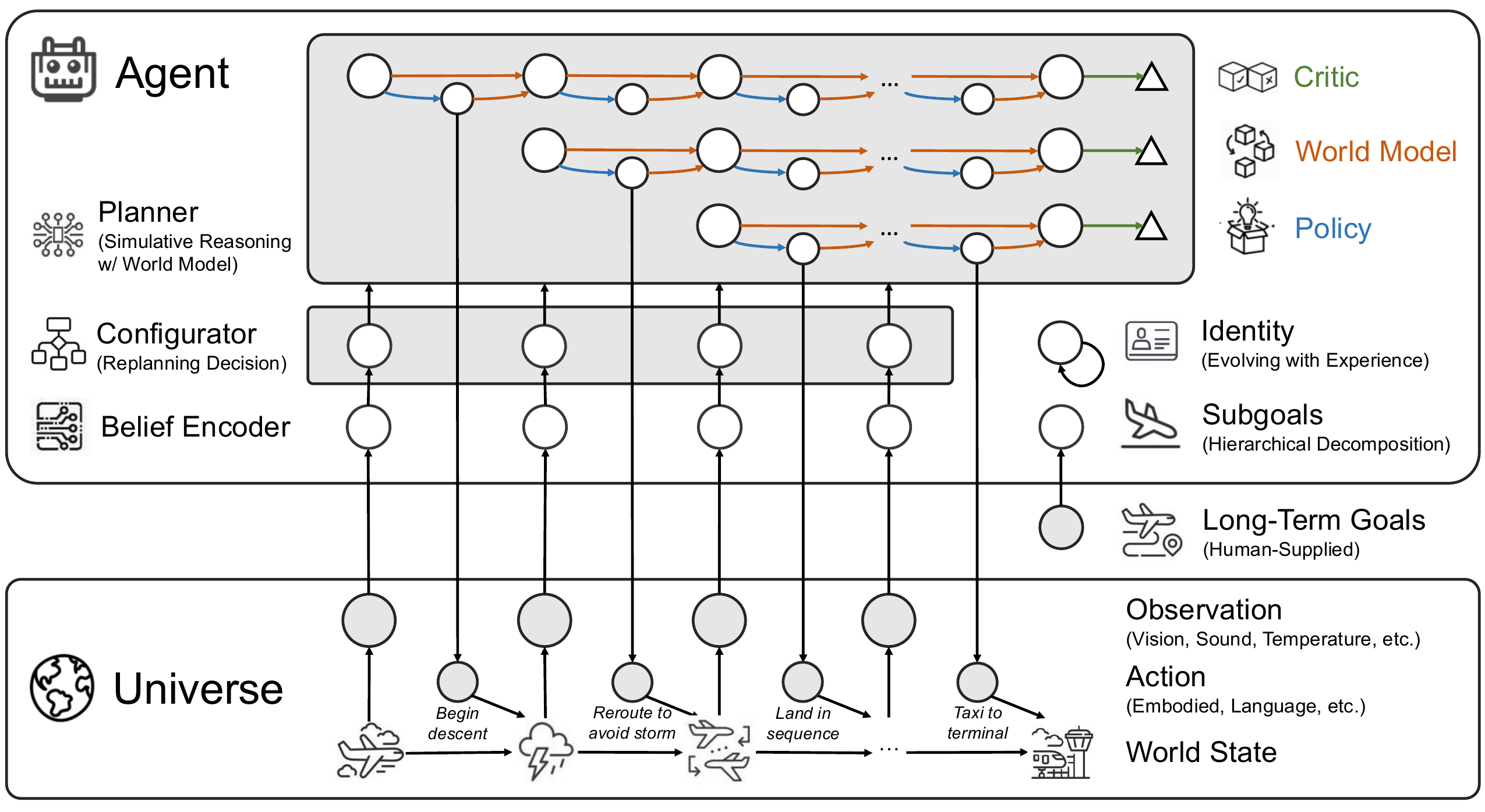}
    \caption{\small The GIC Agent Model architecture, illustrated with the aircraft pilot use case. \textbf{(Bottom)} The universe emits observations and receives actions from the agent. \textbf{(Top)} The agent processes observations through a \emph{belief encoder} to form belief states, conditioned on an evolving \emph{identity} and hierarchically decomposed \emph{subgoals}. The \emph{configurator} (System~III) decides at each step whether to invoke the planner or act directly. When planning is invoked, the \emph{planner} (System~II) simulates candidate trajectories: the \emph{world model} predicts future states, the \emph{policy} proposes candidate actions, and the \emph{critic} evaluates expected long-term value. The best plan is executed through the agent's actions (System~I).}
    \label{fig:gic-architecture}
\end{figure}

The GIC architecture, as illustrated in Figure~\ref{fig:gic-architecture}, consists of six components, each handling a distinct aspect of agency. We describe them in turn.

\paragraph{Belief Encoder ($h$).}
The belief encoder maps the current observation $o_t$ to an internal belief state $\hat{s}_t \sim p_h(\cdot \mid o_t)$, representing the agent's best estimate of the world.
Specifically, as argued in \cite{xing2025critiques}, the belief state is neither just a continuous sensory embedding nor just a text description, but a mixture of discrete tokens (e.g., text) for encoding abstract concepts (e.g., computer code, morality, other agents' goals and capabilities) and continuous embeddings for perceptual details (e.g., fine-grained texture, joint angles)



\paragraph{Goal Decomposer ($\delta$).}
Given the belief state $\hat{s}_t$ and the agent's long-term goal $g$, the goal decomposer produces the active subgoal $g_t \sim p_\delta(\cdot \mid \hat{s}_t, g)$.
Subgoals are ordered by dependency and priority, and revisable as new information arrives.
For the pilot approaching an unfamiliar airport in poor visibility, for example, $\delta$ may decompose the mission into ``execute the instrument approach'' as the immediate subgoal.

\paragraph{Identity Evolver ($\iota$).}
The identity evolver updates the agent's self-model $i_t \sim p_\iota(\cdot \mid \hat{s}_t, i_{t-1})$, capturing capabilities, constraints, affordances, and relationships with other entities.
Identity adapts without retraining, analogous to how a professional revises self-assessment over a busy day without needing to ``rewire their brain.''
The same pilot, after a difficult approach in gusty winds, may revise downward the self-assessed confidence in visual techniques and/or reinforce their mission-driven values ($i_t$), leading to more conservative decisions in general but risk-taking behavior in critical situations going forward.


\paragraph{Configurator ($\kappa$) --- System~III.}
The configurator assesses the current situation and outputs a regulation decision $u_t \sim p_\kappa(\cdot \mid \hat{s}_t, g_t, i_t)$ governing the agent's deliberative mode: construct a new plan, continue or revise an existing one, or skip planning and act directly.
More broadly, it may route among internal capabilities including goal re-decomposition, identity revision, and retreating into learning.
As formalized in \S\ref{subsec:critique-self-regulation}, this learned meta-controller avoids both the waste of always-on planning and the brittleness of fixed workflows.


\paragraph{Simulative Planner ($\pi_f$) --- System~II.}
When planning is invoked, the planner constructs a plan $c_t \sim p_{\pi_f}(\cdot \mid \hat{s}_t, g_t, i_t, u_t)$ by proposing candidate actions, predicting their consequences through the \textbf{world model} $f$, evaluating goal progress through the critic $v$, and choosing the best one while accounting for prediction uncertainty.
The plan encodes a projected trajectory $c_t = (\hat{s}_t, a'_t, \hat{s}_{t+1}, a'_{t+1}, \dots, \hat{s}_T)$.
Predicted future states can be checked against subsequent observations to assess plan validity; planned actions guide execution when anticipated states are encountered 
or when the current state is highly uncertain (e.g., landing aircraft in poor visibility);
and the planning horizon is controllable, enabling hierarchical planning at multiple time scales.
Because simulative reasoning grounds decisions in predicted state transitions rather than pattern-matched responses, it enables \emph{generalizable planning}: the agent reasons about novel situations (e.g., behavior of other agents in shared environments) by composing the world model's predictive knowledge, rather than requiring demonstrations for every new task. As proven in Theorem~\ref{thm:world-model-improve-policy}, this capacity improves any baseline policy, provided the world model is reasonably accurate. 


\paragraph{Actor ($\alpha$) --- System~I.}
The actor selects 
action $a_t \sim p_\alpha(\cdot \mid \hat{s}_t, c_t)$, handling fine-grained reactive patterns that are difficult to encode in structured plans (e.g., the pilot's immediate stall recovery, the instinctive correction on a gust of wind).
In social environments, the actor's action space naturally extends to communicative actions directed at other agents, making multi-agent coordination an emergent consequence of the architecture, rather than requiring a separate mechanism. 

\paragraph{Integration: Three Decision-Making Systems.}
The interplay among these components can be understood through three systems:
\textbf{System~I} (reactive action via the actor $\alpha$) handles routine or urgent decisions where deliberation costs outweigh its benefits;
\textbf{System~II} (simulative planning via $\pi_f$) handles novel or high-stakes situations requiring counterfactual evaluation;
\textbf{System~III} (self-regulation via $\kappa$) governs which mode to engage, whether it be delegating to System~I during calm cruise, activating System~II when weather deteriorates, or rapidly sequencing both when an engine fails during takeoff.


Together, the agent's action distribution decomposes as:
\begin{align}
p_{\text{GIC}}(a_t \mid o_t, g, i_{t-1}) = \sum_{\substack{g_t,\, i_t \\ u_t,\, c_t}} 
&\underbrace{p_\alpha(a_t \mid \hat{s}_t, c_t)}_{\scriptsize \shortstack{actor \\ (System~I)}}\,
\underbrace{p_{\pi_f}(c_t \mid \hat{s}_t, g_t, i_t, u_t)}_{\scriptsize \shortstack{planner \\ (System~II)}}\,
\underbrace{p_\kappa(u_t \mid \hat{s}_t, g_t, i_t)}_{\scriptsize \shortstack{configurator \\ (System~III)}} \\
&\underbrace{p_\iota(i_t \mid \hat{s}_t, i_{t-1})}_{\scriptsize \shortstack{identity\\evolution}}\,
\underbrace{p_\delta(g_t \mid \hat{s}_t, g)}_{\scriptsize \shortstack{goal\\decomposition}}\,
\underbrace{p_h(\hat{s}_t \mid o_t).}_{\scriptsize \shortstack{belief \\ encoder}}
\label{eq:gic-agent-model}
\end{align}
This decomposition defines the variable structure but does not prescribe how each component reasons internally. 
Note that in Equation~\ref{eq:gic-agent-model}, the world model $f$ appears only as the simulator that the planner $\pi_f$ queries, but is not one of its factors. 
The six components above constitute the agent model, with input--output signatures defined over observations, goals, identity, and actions, and are trained to act. The world model $f$ is trained separately on next-state prediction alone, and no gradient from the agent's reward objective flows into its parameters (\S\ref{subsec:critique-learning}). The agent model thus \emph{consults} the world model rather than containing it.
This separation, however, does not preclude the world model and the agent model from working together in a single end-to-end system: while their parameters are disjoint, each set may be updated only by its own objective, and the coupling occurs exclusively through exchange of activations and outputs. 
GIC thus demonstrates that the architectural integration motivating recent joint world-action generators~\citep[e.g.,][]{ye2026dreamzero,nvidia2026cosmos3} is fully compatible with maintaining the functional and training separation on which sound diagnosis and safety analysis depend.

Furthermore, the conditional independence structure among GIC's variables (e.g., the actor depends on the current plan but not on the raw goal; the planner depends on belief state, goal, and identity but not on the configurator's internal state) suggests that structured attention patterns reflecting these graphical constraints may preserve accuracy while substantially reducing computational overhead compared to flat, full-attention architectures.
While the formulation shows a single configurator decision $u_t$ per step, it generalizes to iterative refinement through multiple rounds. Overall, GIC represents a general-purpose architecture for generating intentional, goal-directed behavior across diverse environments, from language-based reasoning, to embodied interaction, and to multi-agent coordination. Detailed architectural choices, including specific end-to-end and attention designs, are the subject of companion and future work~\cite{deng2026generalagenticplanningsimulative,deng2026sr2am}.

\subsection{Training the Agent Model}
\label{sec:am-training}

It should be clear from the pilot example above that no single training paradigm suffices for developing full genuine agency, whether it be self-supervision, demonstration, or reinforcement learning: a pilot who has only read manuals cannot fly; one who only imitates the instructor cannot handle dynamic situations; and one who only learns by trial-and-error will crash many a plane. 
GIC training follows a divide-and-conquer approach across three phases:


\paragraph{Phase 1: Component Pretraining (Ground School)}
The process begins with pretraining for the agent model and the world model as two parallel models with shared ancestry but divergent objectives.
The agent model's reasoning components are initialized from a pretrained LLM, which remains one of the most effective means of internalizing "book knowledge" (e.g., concepts, procedures, conventions, and jargons of its operating domains) that form the basis for the model's abstract reasoning capabilities. For a pilot, this corresponds to the ground school, where the student studies aerodynamics, meteorology, and ATC procedures, but this is not the simulator.
The world model is trained separately using the generative latent prediction (GLP) architecture~\cite{xing2025critiques}, which may likewise start from a pretrained LLM as backbone but extend it to multimodal next-state prediction on richer observation data (e.g., video, proprioception) via self-supervised learning; this is the simulator being built and calibrated.
The two models may thus descend from the same LLM ancestry, but are pretrained as separate components: next-state prediction loss shapes the world model, goal-directed signals shape the agent model (\S\ref{subsec:critique-learning}). The two models meet only at activations, while their parameters are disjoint, and each is trained by its own signal.
Additionally, a critic is pretrained on reward-labeled data for state evaluation, and the policy is initialized on demonstration data (e.g., embodied or language actions) to seed the action distribution.
This phase builds the conceptual vocabulary all subsequent learning draws from, without operational experience.

\paragraph{Phase 2: Simulative RL (Simulator Hours)}
Once the world model $f$ is sufficiently accurate, the agent learns by generating hypothetical trajectories within $f$ and training via reinforcement learning, without costly real-environment interaction.
As formalized in Theorem~\ref{thm:mixture-beats-real-only}, a mixture of simulated and real experience dominates real-only training, up to a slack term from the world model's quality.
Within this sandbox, the agent builds reactive competence (System~I), deliberate planning ability (System~II), and the configurator (System~III).
This is analogous to the pilot's simulator hours: practicing emergencies, severe weather, and coordinated formation approaches with simulated wingmen, in scenarios too dangerous to stage in real flight.


\paragraph{Phase 3: Real-World Deployment and Refinement (First Flights).}
Subsequent deployment in the real world refines the world model to correct simulation-reality gaps, sharpens the configurator's regulation decisions, updates the policy to exploit dynamics not yet captured by the simulator, and evolves identity through direct performance feedback (Theorem~\ref{thm:fast-slow-dominates}). 
This corresponds to the pilot's transition to real aircraft, adapting to G-forces and fatigue, while coordinating with actual air traffic controllers and teammates.
 
A key strength of GIC is that different components leverage different training signals, leading to more efficient use of training data: the world model uses self-supervised prediction; the critic uses temporal-difference learning on reward-labeled experience; the configurator is refined via RL to maximize task success while minimizing computational expenditure; identity evolution can be supervised by measuring iterative improvement.
In the fully realized vision, the configurator governs not only inference-time deliberation but also the scheduling of the agent's own learning, deciding when to act, when to retreat into simulation for offline practice, when to update the world model, and when to revise its self-model.
Such an agent, autonomously interleaving execution and self-improvement, is qualitatively different from one frozen after deployment.


\subsection{Inference by the Agent Model}
\label{sec:am-inference}

At deployment, a trained GIC agent model operates as a persistent, self-regulating system rather than resetting between interactions.
Specifically, the agent receives an 
overall goal $g$ (e.g., flying to a city, winning a battle)
and initial identity $i_0$, decomposes $g$ into subgoals, and begins execution, revising the decomposition as new information arrives.
For each active subgoal, the configurator continuously assesses the belief state and decides whether to construct a new plan, continue a cached plan, or act directly. 
In multi-agent settings, communication and coordination are treated as actions within the agent's standard repertoire, as established in the actor's action space (\S\ref{sec:gic-agent-model}), and are therefore subject to the same planning and regulation framework as any other action.
Meanwhile, simulative reasoning over communicative and/or coordinative action would require a nested ``super world model'' that contains many (typically much simplified) models of other agents, each with their own (also simplified) world models, goals, identities, and other behaviors. This allows the consequences of communication (e.g., whether a teammate will comply, misunderstand, or act independently) to be predicted and evaluated.

During low-urgency periods, deeper routines may activate: updating the world model from recent experience, running simulative training on identified weaknesses, and revising goal decomposition strategies.
The configurator serves as meta-controller for these processes, deciding which self-improvement activities to prioritize given available time and resources.
The defining characteristic is persistent operation with minimal external intervention, whether it be planning and acting during active periods, reflecting and training during rest, or adjusting its self-model as experience accumulates --- all without requiring the external orchestration that current systems depend on.
In this mode of operation, inference and learning are not separate phases but a single process of \emph{continuous learning}: like humans, who constantly perform activities and constantly learn from them, the agent never graduates into pure execution. The capacity to interleave the two autonomously is itself a hallmark of genuine agency.


\subsection{Evaluation of the Agent Model}

Evaluating agentive systems, such as the GIC agent model, requires going beyond task success on fixed benchmarks. We propose evaluation along three complementary dimensions: \textbf{P}erformance, \textbf{E}fficiency, and \textbf{G}rowth (\textbf{PEG}), each targeting different agentive capabilities.


\paragraph{Performance}
Task success should reflect generalizable reasoning rather than narrow domain competence.
Long-horizon tasks requiring hierarchical goal decomposition (e.g., research problems decomposing into literature review, hypothesis formation, experimental design, and synthesis), tasks in diverse environments testing transfer, and tasks with stochastic or multi-agent elements requiring adaptive planning are all more diagnostic than single-turn benchmarks.
Specifically, different task types can isolate different GIC capabilities.
Goal decomposition is tested by tasks where subgoal ordering is critical and errors compound (e.g., cooking a meal, coordinating a group activity).
Identity evolution is tested by environment transfer: the agent is deployed in a new domain and evaluated on how quickly and accurately it adapts.
Simulative reasoning is tested by tasks that reactive policies find difficult, such as those requiring satisfaction of multiple constraints and multiple steps of reasoning before reaching the goal (e.g., multi-constraint or multi-hop web navigation).
Reactive execution is tested by tasks demanding dense, fine-grained interaction with the real world (e.g., object manipulation, open-ended dialogue).
Evaluating these in concert reveals whether the architecture produces coherent agentive behavior, not just competence on any single axis.


\paragraph{Efficiency}
Metrics such as decision latency, computational expenditure, interaction length, and time-to-completion test the configurator's ability to invest deliberation where it helps and skip it where it does not.
Evaluation should report not just average efficiency but the \emph{distribution} of effort across decisions, testing whether the agent allocates resources intelligently.
This is not to diminish the importance of scaling model parameter or inference compute, but rather to ask how smart the scaling approach is.
Concrete ratios that test the configurator's compute-routing ability include accuracy per unit of reasoning cost (e.g., number of thinking tokens, simulation steps, or FLOPs) and planning frequency (how often the configurator invokes System~II deliberation vs.\ System~I reactive execution).
Ideally, evaluation would also measure how well the agent's compute allocation correlates with task difficulty (e.g., an agent that thinks harder on harder problems and acts reflexively on easy ones is exhibiting genuine self-regulation), though this requires a principled definition of difficulty, which remains an open problem in its own right.


\paragraph{Growth}
Arguably the most distinctive dimension: this measures not just initial competence but the learning curve, and is what ultimately separates an agentive system from a fixed-at-deployment tool.
We propose three concrete measures.
First, \emph{learning efficiency}: given the same repository of real-world experience, what level of performance can the agent extract? This tests the quality of the learning mechanism itself.
Second, \emph{self-directed exploration}: given the same budget for real-world interaction, what performance does the agent achieve? This tests the agent's ability to schedule and prioritize its own learning, rather than relying on externally curated curricula.
Third, \emph{learning transfer}: given a fixed amount of learning on in-distribution training tasks, how well does that improvement generalize to out-of-distribution tasks?

Together, PEG targets all five capabilities central to the agentive spectrum: Performance isolates goal decomposition, identity evolution, simulative reasoning, and reactive execution through targeted task design; Efficiency tests self-regulation through compute-allocation analysis; and Growth measures self-directed learning through controlled experience budgets. Our preliminary results~\cite{deng2026generalagenticplanningsimulative,deng2026sr2am} provide initial evidence along the Performance and Efficiency dimensions; Growth evaluation remains an important direction for future work.

\subsection{Data Requirements}
\label{sec:am-data}

Training a GIC agent model requires data reflecting the full range of experience relevant to agency.
A key insight is that different data sources contribute at different levels of the hierarchy, dramatically improving data efficiency.
Indeed, GIC is able to leverage all the traditional data sources: \textbf{observation-only data} (i.e., full sensory experience and book knowledge) for training the world model, \textbf{reward-labeled data} (i.e., trajectories annotated with outcome assessments) for training the critic or evaluator functions, and \textbf{action-labeled demonstration data} (i.e., expert trajectories with action annotations) for seeding the policy with behavioral priors. 

Perhaps more importantly, GIC can make use a new type of \textbf{goal-oriented data}, which record extended, purposeful activity annotated with the goal that organizes the entire sequence.
Consider a video capturing someone leaving an apartment, taking an elevator, hailing a cab, and arriving at an airport.
Each action in isolation appears disconnected; knowing the goal ``fly to Paris'', however, retroactively structures the full trajectory into a coherent plan with identifiable subgoals (e.g., leave home, reach the airport, board the flight) and contingencies (e.g., the cab is delayed, so switch to the subway). 
The same principle applies to multi-agent activity: a recording of a team coordinating a search-and-rescue operation becomes structured once the shared goal, each participant's role and their individual intentions are annotated.
With such goal annotation, even a noisy stream of activities becomes a viable training signal for multi-scale planning: the closer the trajectory is to the goal, the more the preceding actions are associated with task success. 
As this category connects the agent's low-level action to its high-level planning capacity, we believe that curating and scaling goal-oriented datasets is among the highest-leverage investments for training general-purpose agent models. 

A crucial advantage of this data hierarchy is that different sources train different levels of the behavioral distribution, without needing a monolithic dataset covering all aspects simultaneously. Many capabilities (e.g., social norms, coordination strategies, and mental states) 
are accessible only through language data, while only directly embodied skills require physical data, which can often be obtained in controlled or simulated environments.

\subsection{Safety Considerations}
\label{sec:am-safety}

An agent model that maintains persistent goals, evolves its identity, and learns autonomously raises legitimate safety concerns.
Bostrom~\cite{bostrom2014superintelligence} warns of instrumental subgoals (self-preservation, resource acquisition) overriding human control;
Amodei et al.~\cite{amodei2016concrete} identify concrete failure modes (e.g., reward hacking, unsafe exploration, distributional shift);
Russell~\cite{russell2019human} raises the shutdown problem (agents resisting correction).
These concerns are particularly relevant to systems that internalize more of their own behavioral organization.
 

We argue that GIC is structurally well-positioned to address them, because harmful behavior decomposes entirely into two categories: \textbf{goal misspecification} (i.e., the human supplied the wrong objective) and \textbf{component imperfection} (i.e., a module made a mistake while pursuing the goal).
The overall goal $g$ is exogenous, leaving no mechanism for GIC to generate its own terminal objectives. 
Goal decomposition $\delta$ produces subgoals evaluated instrumentally against $g$; a harmful subgoal reflects a poorly trained $\delta$, not emergent fundamental misalignment.
Identity $i_t$ captures capabilities, constraints, and instrumental dispositions such as values and morals (\S\ref{subsec:identity-critique}), but these are subordinate to the exogenous goal $g$ rather than substituting independent terminal objectives (``I prioritize safety in service of the mission'' is categorically different from ``I want self-preservation for its own sake'').
The world model $f$ may predict incorrectly, but these are prediction errors, not value problems.
The configurator $\kappa$ regulates \emph{how} to reason, not \emph{what} to pursue.
Every component is instrumental, inspectable, and improvable; for a sufficiently well-trained system, harmful behavior converges to zero \emph{unless the goal itself is wrong}.

Through this lens, each specific concern finds a concrete diagnosis.
If self-preservation is not useful for $g$, a well-trained $\delta$ should not pursue it; if it does, that is a training error in $\delta$ or $f$. 
Such a mistake is identifiable because $\delta$'s subgoals are explicitly modeled and thus auditable. 
The reason instrumental subgoals appear particularly formidable to safety literature may be that it is studied in the context of monolithic systems, where dangerous subgoals may emerge silently within opaque representations; GIC reduces it to a standard model-debugging problem by exposing the relevant decisions as inspectable outputs.
Reward hacking traces to a misspecified reward function, unsafe exploration to an under-trained configurator, distributional shift to an inaccurate world model, each diagnosable and addressable within the modular architecture.
An agent whose only terminal goal is human-supplied has no intrinsic reason to resist goal revision or shut-down, provided $\delta$ does not erroneously treat self-continuation as instrumental.

Indeed, beyond convergence towards safety, the GIC architecture offers a practical advantage that monolithic systems lack: \emph{layered transparency}.
Because each capability deemed important to agency is realized as an explicit, interpretable capability rather than an emergent property of an opaque system, GIC provides natural checkpoints for human oversight at every layer.
Goal decomposition $\delta$ can be audited to detect undesirable instrumental subgoals $g_t$ and correct them before execution.
Identity evolution $\iota$ can be monitored over time to verify that an appropriate self-model $i_t$ is developing, and to surgically remove any component deemed dangerous.
The predicted futures by the world model $f$ and decisions produced by simulative planner $\pi_f$ can be inspected for consistency with reality and with safety constraints, enabling targeted correction of the agent's decision basis.
Decisions by the configurator $\kappa$ can be audited to verify that deliberation is allocated proportionally to task importance and complexity.
And self-directed learning decisions and progress can be reviewed to not only identify gaps in the agent's competence, but also steer the learning trajectory through targeted reinforcement or correction.

This layered auditability directly addresses commonly raised concerns such as emergent self-goals and the spontaneous emergence of agency (e.g., self-awareness, self-preservation drives).
In GIC, the capabilities most likely to give rise to such concerns (e.g., self-managed goal decomposition, self-modeling through identity, self-regulation through the configurator, and self-improvement through learning) are not latent properties that might or might not emerge; they are \emph{internalized modules} whose development can be monitored and regulated as they become relevant.
Rather than waiting for these capabilities to appear within a black box in ways that are uncontrollable and opaque, GIC makes them visible, auditable, and correctable by construction.


A natural objection may still remain: even if failures are attributable to component imperfection, auditable, and correctable, the system will make mistakes during training, and some may be harmful.
This is, however, true of every learning system, including human professionals.
Pilots crash during training; the response was not to ban pilot training but to develop simulators, staged curricula, instructor oversight, and rigorous incident investigation.
Aviation became the safest mode of transport through iterative improvement within structured risk management, not prohibition.
GIC embodies the same logic: the agent trains primarily in the world model before real deployment; mistakes during simulative training are confined to a safe sandbox; the modular architecture enables targeted diagnosis at the component level.
The relevant question is not whether risk exists during learning, but whether the architecture makes it manageable and decreasing.
The alternative of forgoing autonomous agent models is unrealistic, as the capabilities they offer are genuinely useful, and the aspiration to build them is as old as the field itself.
The choice is whether they are developed within transparent architectures where failures can be isolated and corrected, or within opaque ones where they cannot.
From this perspective, building agents with the right architecture is itself a safety intervention.

\section{Conclusion}

We have set out to examine three fundamental questions: \emph{What on earth is an agent? What constitutes genuine agency? And how should we build such an agent model of practical and general utility?}
Our intent is not to offer definitive answers, but to inspire deeper reflection on questions the field may have too often taken for granted. 

We argue that an agent model is not about the accumulation of external scaffolding, but about internalizing the core characteristics of genuine agency (e.g., goal-oriented action, adaptive identity, self-regulated deliberation, autonomous learning, and emergent social participation) into a single, standalone system; current paradigms and efforts toward this end remain primitive. 
The distinction between \emph{agentic} systems, which execute tasks through externally orchestrated tools and workflows, and \emph{agentive} systems, which derive their capabilities endogenously, is not merely technical, but defines the boundary between systems confined to prescribed production lines and those capable of operating in the open world. 

It is our hope that, by offering critical, but analytical and constructive dissections of some of the most popular practices in building agentic systems, and by presenting our alternative proposal, we can spark further advancements in both theory and implementations of stronger agent models. The GIC architecture we have presented, which combines goal decomposition, identity evolution, simulative reasoning, self-regulation, and self-directed learning, paired with a separately learned world model (as developed as partial prototypes in our companion work~\cite{deng2026generalagenticplanningsimulative,deng2026sr2am}), offers, we believe, a principled and credible path toward the characteristics of genuine agency outlined above. 

Looking ahead, the GIC framework opens several promising directions: scaling from single-agent to multi-agent modeling (e.g., collective behaviors of a business, a society, consequences to public health), extending interaction across different time scales (e.g., from milliseconds to millennia) and modalities, and ultimately enabling autonomous, perpetual learning in open-ended environments. 
We envision agent models becoming useful not only for achieving goals directly, but also for simulating intelligent behaviors as part of broader applications, whether it be scientific research, personnel training, or complex operational planning. 
For these purposes, we believe that frameworks like GIC, with its multi-layer abstraction, empirical scalability, and structural approach to safety, offer a compelling foundation for the development of robust and general-purpose AI.

\clearpage

\bibliographystyle{plain} 
\bibliography{refs}

\appendix

\section{Detailed Restatement and Proof for Theorem~\ref{thm:fast-slow-dominates}}
\label{appendix:fast-slow-dominates}

\newcounter{theoremSaved}
\setcounter{theoremSaved}{\value{theorem}}
\setcounter{theorem}{0}

\begin{theorem}[Fast-Slow Learning Dominates Slow-Only Learning, up to Identity Revision Quality (Restated)]
\label{thm:fast-slow-dominates-restate}
Consider an agent operating over $K$ rounds. Each round $k$ consists of a slow update producing a base policy, followed by $N_k$ steps of interaction with the environment. The slow-only and fast-slow settings induce two base-policy sequences, $\{\pi^{\mathrm{S}}_k\}$ and $\{\pi^{\mathrm{F}}_k\}$, sharing the initialization $\pi^{\mathrm{S}}_1 = \pi^{\mathrm{F}}_1 = \pi_1$ and updated each round from their own experience (Equation~\ref{eq:cross-round-advantage}); they coincide in round~$1$ and may diverge thereafter, since each trains on the experience generated under its own identity schedule. We write $\pi_{k, i}$ for a base policy conditioned on self-model $i$. Let $V^{g}_{\pi, f}$ denote the expected discounted return of policy $\pi$ in the world model $f$, and let $i^*_t := \arg\max_{i \in \mathcal{I}} V^{g}_{\pi^{\mathrm{F}}_{k,i}, f}(\hat{s}_t)$ denote the value-maximizing self-model for belief state $\hat{s}_t$. In the slow-only setting, the agent executes $\pi^{\mathrm{S}}_{k, i_0}$ throughout each round. In the fast-slow setting, the identity evolver $\iota$ produces a revised self-model $i_t \sim p_\iota(\cdot \mid \hat{s}_t, i_{t-1})$ at each step, so the agent executes $\pi^{\mathrm{F}}_{k, i_t}$.

Define the cumulative regret of the slow-only agent as:
\begin{equation}
    \emph{Regret}^{\emph{std}}_K = \sum_{k=1}^{K} \sum_{t=1}^{N_k} \left[ V^{g}_{\pi^*_{i^*_t}, f}(\hat{s}_t) - V^{g}_{\pi^{\mathrm{S}}_{k, i_0}, f}(\hat{s}_t) \right],
    \label{eq:regret-standard}
\end{equation}
and the cumulative regret of the fast-slow agent as:
\begin{equation}
    \emph{Regret}^{\emph{fast-slow}}_K = \sum_{k=1}^{K} \sum_{t=1}^{N_k} \left[ V^{g}_{\pi^*_{i^*_t}, f}(\hat{s}_t) - V^{g}_{\pi^{\mathrm{F}}_{k, i_t}, f}(\hat{s}_t) \right].
    \label{eq:regret-fastslow}
\end{equation}

Under Assumptions A1 and A2 below, define the per-step expected value improvement from identity revision as:
\begin{equation}
    \bar{\varepsilon} := \inf_{k, t} \; \mathbb{E}\left[V^{g}_{\pi^{\mathrm{F}}_{k, i_t}, f} - V^{g}_{\pi^{\mathrm{F}}_{k, i_0}, f}\right] > 0,
    \label{eq:per-step-gain}
\end{equation}
where positivity follows from A1. Then the following bound holds:
$$
    \emph{Regret}^{\emph{fast-slow}}_K \le \emph{Regret}^{\emph{std}}_K - \underbrace{\sum_{k=1}^{K} N_k \bar{\varepsilon}}_{\emph{within-round gain}} - \underbrace{\sum_{k=2}^{K} N_k \eta_k}_{\emph{cross-round compounding}},
$$
where $\eta_k \ge 0$ is the cross-round advantage defined in Equation~\ref{eq:cross-round-advantage}.
\end{theorem}

\setcounter{theorem}{\value{theoremSaved}}

\textbf{Assumption A1 (identity revisions improve the self-model and better self-models produce better decisions).}

Let $d(i, i')$ be a divergence measure between self-models.

\emph{Part (a): identity revision closes the gap.} For some $\varepsilon > 0$ and $\delta_1 \in (0, 1/2)$, at each step $t$ within round $k$:
\begin{equation}
    \Pr\big(d(i_0, i^*_t) - d(i_t, i^*_t) \ge \varepsilon\big) \ge 1 - \delta_1,
    \label{eq:a1a}
\end{equation}
with bounded degradation on the complementary event: $d(i_t, i^*_t) - d(i_0, i^*_t) \le \varepsilon$ almost surely.

\emph{Part (b): closer self-models yield higher value with high probability.} For some $\delta_2 \in (0, 1/2)$ and value gain $\lambda > 0$:
\begin{equation}
    \Pr\big(V^{g}_{\pi^{\mathrm{F}}_{k, i_t}, f}(\hat{s}_t) - V^{g}_{\pi^{\mathrm{F}}_{k, i_0}, f}(\hat{s}_t) \ge \lambda \;\big|\; d(i_t, i^*_t) < d(i_0, i^*_t)\big) \ge 1 - \delta_2,
    \label{eq:a1b}
\end{equation}
with bounded degradation: $V^{g}_{\pi^{\mathrm{F}}_{k, i_0}, f}(\hat{s}_t) - V^{g}_{\pi^{\mathrm{F}}_{k, i_t}, f}(\hat{s}_t) \le B$ almost surely on the complementary event, for some $B > 0$.

\textbf{Assumption A2 (the slow update operator is monotone in base- and data-generating-policy quality).}

Let $\mathcal{U}$ denote the slow update operator, and let $\bar{V}(\pi) := \mathbb{E}_{\hat{s}} \left[ V^{g}_{\pi, f}(\hat{s}) \right]$ denote the expected performance of policy $\pi$ in the world model.

\emph{Part (a): joint monotonicity.} The update operator $\mathcal{U}$ satisfies: for any base policies $\pi, \tilde{\pi}$ and behavioral policies $\pi_A, \pi_B$,
\begin{equation}
    \bar{V}(\pi) \ge \bar{V}(\tilde{\pi}) \;\;\text{and}\;\; \bar{V}(\pi_A) \ge \bar{V}(\pi_B) \;\;\Longrightarrow\;\; \bar{V}\!\left(\mathcal{U}(\pi, \mathcal{D}^{\pi_A})_{i_0}\right) \ge \bar{V}\!\left(\mathcal{U}(\tilde{\pi}, \mathcal{D}^{\pi_B})_{i_0}\right),
    \label{eq:a2a}
\end{equation}
where $\mathcal{D}^{\pi_A}, \mathcal{D}^{\pi_B}$ denote experience collected under $\pi_A, \pi_B$. The single-base case ($\pi = \tilde{\pi}$) recovers monotonicity in behavioral-policy quality alone. The output policies are evaluated at identity $i_0$ because the slow update resets the identity to its initial value at the start of each round.

\emph{Part (b): the identity-revised policy is the stronger behavioral policy.} From A1 and the definition of $\pi^{\mathrm{F}}_{k, i_t}$:
\begin{equation}
    \bar{V}(\pi^{\mathrm{F}}_{k, i_t}) \ge \bar{V}(\pi^{\mathrm{F}}_{k, i_0}).
    \label{eq:a2b}
\end{equation}

With the base-policy sequences $\pi^{\mathrm{F}}_{k+1} = \mathcal{U}(\pi^{\mathrm{F}}_k, \mathcal{D}^{\pi^{\mathrm{F}}_{k, i_t}})$ and $\pi^{\mathrm{S}}_{k+1} = \mathcal{U}(\pi^{\mathrm{S}}_k, \mathcal{D}^{\pi^{\mathrm{S}}_{k, i_0}})$, both from $\pi^{\mathrm{F}}_1 = \pi^{\mathrm{S}}_1 = \pi_1$, define the cross-round advantage as the cumulative base-policy gap:
\begin{equation}
    \eta_k := \bar{V}(\pi^{\mathrm{F}}_{k, i_0}) - \bar{V}(\pi^{\mathrm{S}}_{k, i_0}), \qquad \eta_1 = 0.
    \label{eq:cross-round-advantage}
\end{equation}
Under Parts (a) and (b), $\eta_k \ge 0$ for all $k$. Because the two sequences diverge after round~1, this is established by carrying the advantage over from round to round (an induction in Step~3 of the proof) rather than by a single application of Part~(a).

\begin{explanation}
A1 and A2 operate on quantities the agent designer can verify independently. A1(a) asks that the identity evolver $\iota$ moves the self-model toward the value-maximizing $i^*_t$, which is its training objective. A1(b) asks that decisions conditioned on self-models closer to $i^*_t$ tend to produce higher value, which is the fundamental premise of conditioning on identity at all.

A2 relocates the cross-round assumption from the value function to the update operator $\mathcal{U}$. Its single-base form ($\pi = \tilde{\pi}$) is a structural property satisfied by many standard methods, including conservative policy iteration~\cite{kakade2002approximately}, natural policy gradient~\cite{kakade2001natural}, and trust-region methods~\cite{schulman2015trust}; the joint form stated in Part~(a) is the natural extension to differing base policies, in the same spirit and testable the same way. We require the joint form because identity revision makes the two agents collect different experience, so their base policies genuinely diverge after round~1 and the cross-round comparison is between policies trained from different bases. Part~(b) is not an independent assumption but a consequence of A1: identity-revised interaction, by conditioning on a self-model closer to $i^*_t$, yields higher expected return than fixed-identity interaction, so $\pi^{\mathrm{F}}_{k, i_t}$ is the stronger behavioral policy.

The non-negativity $\eta_k \ge 0$ then follows by carrying the advantage over (Step~3): if the fast-slow base policy leads the slow-only one entering round~$k$, then within round~$k$ it both starts from the stronger base and collects stronger experience, so by Part~(a) it still leads entering round~$k+1$. This carry-over preserves the advantage but is not required to grow it: A2 asks only that $\eta_k \ge 0$, so the slow update cannot erase what fast adaptation has gained but need not amplify it. The condition is testable in practice: given a specific choice of $\mathcal{U}$ (e.g., PPO, SAC, or even supervised fine-tuning on filtered experience), one can verify monotonicity by comparing the output policies when trained from base policies and rollouts of differing quality.
\end{explanation}

\medskip
\begin{proof}
The proof proceeds in three steps: establishing the per-step gain from identity revision (Step~1), aggregating the within-round advantage (Step~2), and carrying the cross-round advantage over (Step~3).

\noindent
\paragraph{Step 1: Per-step value improvement from identity revision.}
Fix any round $k$ and step $t$. Define the per-step value difference at the fast-slow base policy:
$$
\Delta_t := V^{g}_{\pi^{\mathrm{F}}_{k, i_t}, f}(\hat{s}_t) - V^{g}_{\pi^{\mathrm{F}}_{k, i_0}, f}(\hat{s}_t).
$$
We decompose the expectation of $\Delta_t$ by conditioning on whether A1(a) and A1(b) jointly succeed. Let $E_1$ denote the event that identity revision closes the gap by at least $\varepsilon$ (Inequality~\ref{eq:a1a}), and let $E_2$ denote the event that the closer self-model yields a value improvement of at least $\lambda$ (Inequality~\ref{eq:a1b}). Then:
\begin{align*}
    \mathbb{E}[\Delta_t] &= \mathbb{E}[\Delta_t \mid E_1 \cap E_2]\, \Pr(E_1 \cap E_2) + \mathbb{E}[\Delta_t \mid \overline{E_1 \cap E_2}]\, \Pr(\overline{E_1 \cap E_2}).
\end{align*}
By A1, the joint event $E_1 \cap E_2$ occurs with probability at least $(1-\delta_1)(1-\delta_2)$. On this event, $\Delta_t \ge \lambda$ by Inequality~\ref{eq:a1b}. On the complementary event, the bounded degradation conditions in A1 guarantee $\Delta_t \ge -B$. Setting $\delta := \delta_1 + \delta_2 - \delta_1 \delta_2 < 1$, we obtain:
\begin{equation}
    \mathbb{E}[\Delta_t] \ge (1 - \delta) \lambda - \delta B.
    \label{eq:per-step-lower-bound}
\end{equation}
Since $\delta_1, \delta_2 \in (0, 1/2)$, we have $\delta < 3/4$, and for $\lambda, B$ satisfying $(1-\delta)\lambda > \delta B$ (which is ensured when the identity evolver is better than random, i.e., $\lambda / B > \delta / (1-\delta)$), the right-hand side is strictly positive. Defining:
$$
\bar{\varepsilon} := \inf_{k,t}\; \mathbb{E}[\Delta_t] \ge (1-\delta)\lambda - \delta B > 0,
$$
establishes the per-step gain claimed in Equation~\ref{eq:per-step-gain}. The argument uses no property specific to $\pi^{\mathrm{F}}_k$ and holds for any base policy.

\medskip
\noindent
\paragraph{Step 2: Within-round regret reduction.}
Within round $k$, the per-step difference between the two agents' regret is, since the slow-only actor is $\pi^{\mathrm{S}}_{k, i_0}$ and the fast-slow actor is $\pi^{\mathrm{F}}_{k, i_t}$,
\begin{align*}
    &\left[V^{g}_{\pi^*_{i^*_t}, f}(\hat{s}_t) - V^{g}_{\pi^{\mathrm{S}}_{k, i_0}, f}(\hat{s}_t)\right] - \left[V^{g}_{\pi^*_{i^*_t}, f}(\hat{s}_t) - V^{g}_{\pi^{\mathrm{F}}_{k, i_t}, f}(\hat{s}_t)\right]
    = V^{g}_{\pi^{\mathrm{F}}_{k, i_t}, f}(\hat{s}_t) - V^{g}_{\pi^{\mathrm{S}}_{k, i_0}, f}(\hat{s}_t).
\end{align*}
Adding and subtracting $V^{g}_{\pi^{\mathrm{F}}_{k, i_0}, f}(\hat{s}_t)$ splits this into a within-round and a cross-round part:
$$
\underbrace{V^{g}_{\pi^{\mathrm{F}}_{k, i_t}, f}(\hat{s}_t) - V^{g}_{\pi^{\mathrm{F}}_{k, i_0}, f}(\hat{s}_t)}_{=\,\Delta_t \;\text{(within-round)}}
\;+\;
\underbrace{V^{g}_{\pi^{\mathrm{F}}_{k, i_0}, f}(\hat{s}_t) - V^{g}_{\pi^{\mathrm{S}}_{k, i_0}, f}(\hat{s}_t)}_{\text{cross-round base gap}}.
$$
Taking expectations of the within-round part and summing over the $N_k$ steps of round~$k$:
\begin{equation}
    \sum_{t=1}^{N_k} \mathbb{E}[\Delta_t] \ge N_k \bar{\varepsilon},
    \label{eq:within-round-bound}
\end{equation}
which is the within-round contribution to $\mathbb{E}[\text{Regret}^{\text{std}}_k - \text{Regret}^{\text{fast-slow}}_k]$; the remaining cross-round contribution is handled in Step~3. Summing Inequality~\ref{eq:within-round-bound} over all $K$ rounds gives the within-round gain $\sum_{k=1}^{K} N_k \bar{\varepsilon}$, which is available even if no further slow updates ever occur.

\medskip
\noindent
\paragraph{Step 3: Cross-round compounding by carrying the advantage over.}
Summed over steps and rounds, the cross-round part contributes, in expectation, $\sum_{k} N_k \eta_k$ with $\eta_k = \bar{V}(\pi^{\mathrm{F}}_{k, i_0}) - \bar{V}(\pi^{\mathrm{S}}_{k, i_0})$ (Equation~\ref{eq:cross-round-advantage}). It remains to show $\eta_k \ge 0$ for all $k$, which we do by induction: the base-policy advantage is carried over from each round to the next.

\emph{Base case.} $\eta_1 = 0$, since $\pi^{\mathrm{F}}_1 = \pi^{\mathrm{S}}_1 = \pi_1$.

\emph{Inductive step.} Suppose $\eta_k \ge 0$, i.e.\ $\bar{V}(\pi^{\mathrm{F}}_{k, i_0}) \ge \bar{V}(\pi^{\mathrm{S}}_{k, i_0})$. By A2(b) (a consequence of A1), $\bar{V}(\pi^{\mathrm{F}}_{k, i_t}) \ge \bar{V}(\pi^{\mathrm{F}}_{k, i_0})$; chaining with the inductive hypothesis,
$$
\bar{V}(\pi^{\mathrm{F}}_{k, i_t}) \;\ge\; \bar{V}(\pi^{\mathrm{F}}_{k, i_0}) \;\ge\; \bar{V}(\pi^{\mathrm{S}}_{k, i_0}).
$$
Thus entering the slow update, the fast-slow agent both starts from a base policy at least as strong ($\bar{V}(\pi^{\mathrm{F}}_{k, i_0}) \ge \bar{V}(\pi^{\mathrm{S}}_{k, i_0})$) and collects experience under a behavioral policy at least as strong ($\bar{V}(\pi^{\mathrm{F}}_{k, i_t}) \ge \bar{V}(\pi^{\mathrm{S}}_{k, i_0})$). Applying the joint monotonicity of $\mathcal{U}$ (Inequality~\ref{eq:a2a}) to $\big(\pi^{\mathrm{F}}_k,\, \pi^{\mathrm{F}}_{k, i_t}\big)$ versus $\big(\pi^{\mathrm{S}}_k,\, \pi^{\mathrm{S}}_{k, i_0}\big)$ yields
$$
\bar{V}(\pi^{\mathrm{F}}_{k+1, i_0}) \;\ge\; \bar{V}(\pi^{\mathrm{S}}_{k+1, i_0}),
$$
i.e.\ $\eta_{k+1} \ge 0$, completing the induction. The advantage opened in round~1 by identity revision is therefore preserved through every subsequent slow update. Hence each $\eta_k \ge 0$, and the cross-round part contributes $\sum_{k=2}^{K} N_k \eta_k$ (the $k=1$ term vanishes since $\eta_1 = 0$).

\medskip
\noindent
\paragraph{Combining the terms.}
Adding the within-round gain (Step~2) and the cross-round contribution (Step~3), we obtain:
$$
\text{Regret}^{\text{fast-slow}}_K \le \text{Regret}^{\text{std}}_K - \sum_{k=1}^{K} N_k \bar{\varepsilon} - \sum_{k=2}^{K} N_k \eta_k,
$$
which completes the proof. The first subtracted term grows linearly in the total number of interaction steps $\sum_k N_k$; the second adds a non-negative contribution at every round beyond the first, so the cross-round reduction is non-decreasing in $K$. The advantage of fast-slow over slow-only learning thus widens with both longer interactions and more update cycles.
\end{proof}

\section{Proof for Theorem~\ref{thm:world-model-improve-policy}}
\label{appendix:world-model-improve-policy}

\begin{proof}
Given policy $\pi$, recall its state value function in the true environment $\mu$ as $V^g_{\pi, \mu}(s)$ (Equation~\ref{eq:value-function}) and its action-value function:
$$Q^g_{\pi,\mu}(s, a) = \sum_{s'} \left[ r(g, s) + \gamma V^g_{\pi,\mu}(s') \right] p_\mu(s' \mid s, a),$$
which describes the expected discounted reward of choosing action $a$ in state $s$ and following policy $\pi$ thereafter.
Define $V^g_{\pi, f}$ and $Q^g_{\pi, f}$ analogously with respect to the world model $f$.
Then by the Simulation Lemma~\citep{kearns2002near}, for all state-action pairs $(s, a)$, the state value and state-action value differ only by:
\begin{equation*}
    \lvert V^g_{\pi,\mu}(s) - V^g_{\pi,f}(s) \rvert \leq \epsilon_{\text{model}}, \qquad 
    \lvert Q^g_{\pi,\mu}(s, a) - Q^g_{\pi,f}(s, a) \rvert \leq \epsilon_{\text{model}},
\end{equation*}
where $\epsilon_{\text{model}} = \frac{2 \gamma R_{\text{max}} \epsilon}{(1-\gamma)^2}$.
 
\medskip
\noindent 
Further define the advantage function in the true environment $\mu$:
$$A^g_{\pi,\mu}(s, a) = Q^g_{\pi,\mu}(s, a) - V^g_{\pi,\mu}(s),$$ 
which measures how much better action $a$ is compared to simply following $\pi$. 
A similar definition holds for $A^g_{\pi,f}$ under the world model.
 
\medskip
\noindent
Let $\pi^*_f = \argmax_{\pi} V^g_{\pi,f}$ be the optimal policy under the world model (Equation~\ref{eq:world-model-decision-making}). 
Define the mixed decision rule $\pi_{\text{mix}} = \phi(\pi, f, \epsilon)$ as the following:
$$
\pi_{\text{mix}}(s) = \begin{cases}
    \pi^*_f(s) & \text{if $A^g_{\pi,f}(s, \pi^*_f(s)) > 2 \epsilon_{\text{model}}$} \\
    \pi(s) & \text{o.w.}
\end{cases}
$$
In other words, $\pi_{\text{mix}}$ follows the result of world-model-based planning $\pi^*_f$ only when it looks clearly better than $\pi$, leaving a margin $2\epsilon_{\text{model}}$ for model error. 
 
\medskip
\noindent
Now we proceed to show that $V^g_{\pi_{\text{mix}},\mu} \geq V^g_{\pi,\mu}$.
For any $(s, a)$, we can bound:
\begin{align*}
    A^g_{\pi,\mu}(s, a) - A^g_{\pi,f}(s, a) = \big( {\underbrace {Q^g_{\pi,\mu}(s, a) - Q^g_{\pi,f}(s, a)}_{\geq -\epsilon_{\text{model}}}} \big) - \big( {\underbrace {V^g_{\pi,\mu}(s) - V^g_{\pi,f}(s)}_{\geq -\epsilon_{\text{model}}}} \big) 
    \geq -2\epsilon_{\text{model}}.
\end{align*}
Hence, whenever $\pi_{\text{mix}}(s) = \pi^*_f(s)$,
$$
A^g_{\pi, \mu}(s, \pi_{\text{mix}}(s)) \geq A^g_{\pi, f}(s, \pi^*_f(s)) - 2 \epsilon_{\text{model}} > 0.
$$
Otherwise, $\pi_{\text{mix}}(s) = \pi(s)$ and $A^g_{\pi, \mu}(s, \pi_{\text{mix}}(s)) = 0$. Thus, for all $s$, $A^g_{\pi, \mu}(s, \pi_{\text{mix}}(s)) \geq 0$, with strict positivity on any state where switching occurs.  
 
\medskip
\noindent
By the Performance Difference Lemma~\cite{kakade2002approximately}:
$$
V^g_{\pi_{\text{mix}},\mu} - V^g_{\pi,\mu} = \frac{1}{1-\gamma} \mathbb{E}_{s \sim d^{\pi_{\text{mix}}}_\mu} \left[ A^g_{\pi, \mu}\left(s, \pi_{\text{mix}}(s) \right) \right] \geq 0,
$$
where $d^{\pi_{\text{mix}}}_\mu$ is the marginal state distribution induced by policy $\pi_{\text{mix}}$ in environment $\mu$.
The inequality is strict whenever $\pi_{\text{mix}}$ adopts $\pi^*_f$ on a set of states with nonzero probability in $d^{\pi_{\text{mix}}}_\mu$.  
This proves that $V^g_{\pi_{\text{mix}},\mu} \ge V^g_{\pi,\mu}$.
\end{proof}

\section{Proof for Theorem~\ref{thm:horizon-requirement-mpc}}
\label{appendix:horizon-requirement-mpc}

\begin{proof}
Consider the cost function $C_g(s)$ as defining an augmented reward function $\tilde{r}(s, g) = -C_g(s)$. Let $\Tilde{T}$ denote the augmented Bellman operator on $f$ under $\tilde{r}$, namely given value function $V$:
\[
(\tilde{T}V)(s_t) \vcentcolon= \max_{a} \sum_{s_{t+1}} \left[ \tilde{r}(s_t, g) + \gamma V(s_{t+1}) \right] p_f(s_{t+1} \mid s_t, a_t),
\]
And for any policy $\pi$, let $\tilde{T}_\pi$ be its augmented Bellman operator defined as below:
\[
(\tilde{T}V)(s_t) \vcentcolon= \sum_{a_t,s_{t+1}} \left[ \tilde{r}(s_t, g) + \gamma V(s_{t+1}) \right] p_f(s_{t+1} \mid s_t, a_t) p_\pi(a_t \mid s_t).
\]
With $\tilde{\pi}^* = \argmax_\pi \tilde{V}_{\pi,f}$, the values $\tilde{V}^g_{\tilde{\pi}^*,f}$ and $\tilde{V}^g_{\pi,f}$ for $\tilde{r}$ are thus the unique fixed points of $\tilde{T}$ and $\tilde{T}_\pi$, respectively. In other words:
\begin{equation}
    \label{eq:contraction-fixed-points}
    \tilde{V}^g_{\tilde{\pi}^*,f} = \tilde{T} \tilde{V}^g_{\tilde{\pi}^*,f}\, \text{and}\, \tilde{V}^g_{\pi,f} = \tilde{T}_\pi \tilde{V}^g_{\pi,f}.
\end{equation}
Indeed, both $\tilde{T}$ and $\tilde{T}_\pi$ are $\gamma$-contractions in the sup norm~\cite{zhao2025RLBook}. 

\paragraph{Step 1:} Given any bounded value function $V$, let $\pi$ be greedy with respect to $V$ (i.e., $\tilde{T}V = \tilde{T}_\pi V$). We claim that: 
\begin{equation}
\label{eq:approx-greedy-bound}
\lVert \tilde{V}^g_{\tilde{\pi}^*,f} - \tilde{V}^g_{\pi,f} \rVert_\infty \leq \frac{2 \gamma}{1 - \gamma} \lVert \tilde{V}^g_{\tilde{\pi}^*,f} - V \rVert_\infty.
\end{equation}
Indeed, by Equation~\ref{eq:contraction-fixed-points}:
\[
\tilde{V}^g_{\tilde{\pi}^*,f} - \tilde{V}^g_{\pi,f} = \tilde{T} \tilde{V}^g_{\tilde{\pi}^*,f} - \tilde{T}_\pi \tilde{V}^g_{\pi,f}.
\]
Using the greedy condition $\tilde{T}V = \tilde{T}_\pi V$, we have that: 
\[
\tilde{V}^g_{\tilde{\pi}^*,f} - \tilde{V}^g_{\pi,f} = (\tilde{T} \tilde{V}^g_{\tilde{\pi}^*,f} - \tilde{T}V) + (\tilde{T}_\pi V - \tilde{T}_\pi \tilde{V}^g_{\pi,f}).
\]
Taking sup norms and by properties of the $\gamma$-contraction:
\begin{align}
\label{eq:contraction-bound}
    \lVert \tilde{V}^g_{\tilde{\pi}^*,f} - \tilde{V}^g_{\pi,f} \rVert_\infty
    &\leq \gamma \lVert \tilde{V}^g_{\tilde{\pi}^*,f} - V \rVert_\infty + \gamma \lVert V - \tilde{V}^g_{\pi,f} \rVert_\infty.
\end{align}
Now, decompose $V - \tilde{V}^g_{\pi,f} = V - \tilde{V}^g_{\tilde{\pi}^*,f} + \tilde{V}^g_{\tilde{\pi}^*,f} - \tilde{V}^g_{\pi,f}$, then based on the triangle inequality, we also have:
$$
\lVert V - \tilde{V}^g_{\pi,f} \rVert_\infty
\leq \lVert V - \tilde{V}^g_{\tilde{\pi}^*,f} \rVert_\infty + \lVert \tilde{V}^g_{\tilde{\pi}^*,f} - \tilde{V}^g_{\pi,f} \rVert_\infty.
$$
Substituting back into Inequality~\ref{eq:contraction-bound}, we have:
$$
\lVert \tilde{V}^g_{\tilde{\pi}^*,f} - \tilde{V}^g_{\pi,f} \rVert_\infty
\leq 2 \gamma \lVert \tilde{V}^g_{\tilde{\pi}^*,f} - V \rVert_\infty + \gamma \lVert \tilde{V}^g_{\tilde{\pi}^*,f} - \tilde{V}^g_{\pi,f} \rVert_\infty,
$$
which is equivalent to:
$$
\lVert \tilde{V}^g_{\tilde{\pi}^*,f} - \tilde{V}^g_{\pi,f} \rVert_\infty
\leq \frac{2 \gamma}{1 - \gamma} \lVert \tilde{V}^g_{\tilde{\pi}^*,f} - V \rVert_\infty,
$$
proving our claim for Step~1.

\paragraph{Step 2:} Define the value iterate $\hat{V}^{(0)} = 0$ and $\hat{V}^{(K)} = \tilde{T}^K \hat{V}^{(0)}$. Hence $\hat{V}^{(H-1)} = \tilde{T}^{H-1} 0$, which represents the augmented reward of the finite-horizon rollout with zero terminal value. The pure $H$-step MPC policy can therefore be seen as acting greedily with respect to $\hat{V}^{(H-1)}$. In other words:
\[
\tilde{T} \hat{V}^{(H-1)} = \tilde{T}_{\pi^H_{\text{MPC}}} \hat{V}^{(H-1)}.
\]
Therefore, apply Inequality~\ref{eq:approx-greedy-bound} and take $\pi = \pi^H_{\text{MPC}}$:
\begin{equation}
    \label{eq:mpc-approx-greedy-bound}
    \lVert \tilde{V}^g_{\tilde{\pi}^*,f} - \tilde{V}^g_{\pi^H_{\text{MPC}},f} \rVert_\infty
    \leq \frac{2 \gamma}{1 - \gamma} \lVert \tilde{V}^g_{\tilde{\pi}^*,f} - \hat{V}^{(H-1)} \rVert_\infty.
\end{equation}

\paragraph{Step 3:} Since $\tilde{V}^g_{\tilde{\pi}^*,f} = \tilde{T}^{H-1} \tilde{V}^g_{\tilde{\pi}^*,f}$, $\hat{V}^{(H-1)} = \tilde{T}^{H-1} 0$, and $\tilde{T}$ is a $\gamma$-contraction, we have:
\begin{align*}
    \lVert \tilde{V}^g_{\tilde{\pi}^*,f} - \hat{V}^{(H-1)} \rVert_\infty &= \lVert T^{H-1} \tilde{V}^g_{\tilde{\pi}^*,f} - T^{H-1} 0 \rVert_\infty \\
    &\leq \gamma^{H-1} \lVert \tilde{V}^g_{\tilde{\pi}^*,f} \rVert_\infty \\
    &\leq \gamma^{H-1} \frac{C_\text{max}}{1 - \gamma}.
\end{align*}
Substituting this into Inequality~\ref{eq:mpc-approx-greedy-bound} gives:
\begin{equation}
\label{eq:mpc-augmented-bound}
\lVert \tilde{V}^g_{\tilde{\pi}^*,f} - \tilde{V}^g_{\pi^H_{\text{MPC}},f} \rVert_\infty 
\leq \frac{2 \gamma^H C_{\text{max}}}{(1 - \gamma)^2}.
\end{equation}

\paragraph{Step 4:} Because the cost function $C_g$ (and, by extension, the augmented reward $\tilde{r}$) is perfectly aligned with the original reward $r$ (i.e., $\tilde{r}(s, g) = -C_g(s) = r(s, g) - b_g$), for any policy $\pi$:
\[
\tilde{V}_{\pi,f}(s) = \mathbb{E}_{\pi,f}\left[ \sum_{k=0}^\infty \gamma^k (r(s_k, g) - b_g) \mid s_0 = s \right] = V^g_{\pi,f}(s) - \frac{b_g}{1 - \gamma}.
\]
As the constant $\frac{b_g}{1 - \gamma}$ does not depend on $\pi$, maximizing $\tilde{V}_{\pi,f}$ is equivalent to maximizing $V^g_{\pi,f}$, hence $\tilde{\pi}^* = \pi^*$. Moreover, the LHS of Inequality~\ref{eq:mpc-augmented-bound} satisfies:
\begin{align*}
\lVert \tilde{V}^g_{\tilde{\pi}^*,f} - \tilde{V}^g_{\pi^H_{\text{MPC}},f} \rVert_\infty 
&= \lVert (V^g_{\pi^*,f} - \frac{b_g}{1 - \gamma} ) - (V^g_{\pi^H_{\text{MPC}},f} - \frac{b_g}{1 - \gamma}) \rVert_\infty \\
&= \lVert V^g_{\pi^*,f} - V^g_{\pi^H_{\text{MPC}},f} \rVert_\infty
\end{align*}
Hence:
\begin{align*}
\lVert V^g_{\pi^*,f} - V^g_{\pi^H_{\text{MPC}},f} \rVert_\infty
&\leq \frac{2 \gamma^H C_{\text{max}}}{(1 - \gamma)^2}.
\end{align*}

\paragraph{Step 5:}
Given $\epsilon > 0$, to ensure $\lVert V^g_{\pi^*,f} - V^g_{\pi^H_{\text{MPC}},f} \rVert_\infty \leq \epsilon$, we need:
$$
\frac{2 \gamma^H C_{\text{max}}}{(1 - \gamma)^2} \leq \epsilon,
$$
Solving which results in:
\begin{equation}
\label{eq:planning-horizon-scaling}
    H \geq \frac{\log \frac{2 C_{\text{max}}}{\epsilon (1 - \gamma)^2}}{\log \frac{1}{\gamma}}.
\end{equation}
For $\gamma$ close to 1, $\log \frac{1}{\gamma} = \Theta(1 - \gamma)$, so:
\begin{equation}
    H = O\left( \frac{1}{1 - \gamma} \left[ \log \frac{1}{\epsilon} + 2 \log \frac{1}{1 - \gamma} + \log C_{\text{max}} \right] \right).
\end{equation}
If $\gamma$ and $C_{\text{max}}$ are treated as constants, then:
\begin{equation}
    H = O\left(\log \frac{1}{\epsilon}\right),
\end{equation}
Which completes the proof.
\end{proof}

\section{Proof for Theorem~\ref{thm:mixture-beats-real-only}}
\label{appendix:mixture-beats-real-only}

\begin{proof}
Given policy $\pi$, by the Simulation Lemma and the definition of the mixed experience $M_\alpha$, the value of $\pi$ in $M_\alpha$ differs from that in the real environment $\mu$ by the following amount:
\begin{equation}
\label{eq:mixed-environment-simulation}
\lvert V^g_{\pi,M_{\alpha}} - V^g_{\pi,\mu} \rvert \leq C(\gamma, R_\text{max}) \alpha \epsilon,
\end{equation}
where $C(\gamma, R_\text{max}) = \frac{2 \gamma R_\text{max}}{(1 - \gamma)^2}$. On the other hand, $\Pi_{\text{env}}(D_\mu) \subseteq \Pi_{\text{mix}}(D_\mu, D_f)$ by construction, because having access to the world model $f$ and extra simulated experience cannot reduce what one is allowed to compute. As a result:
$$
V^g_{\pi^*_{\text{mix}}, M_\alpha} \geq V^g_{\pi^*_{\text{env}}, M_\alpha}.
$$
By Inequality~\ref{eq:mixed-environment-simulation}, we have:
\begin{align*}
    V^g_{\pi^*_{\text{mix}}, \mu} &\geq V^g_{\pi^*_{\text{mix}}, M_\alpha} - C(\gamma, R_\text{max}) \alpha \epsilon \quad \text{and} \quad
    V^g_{\pi^*_{\text{env}}, M_\alpha} \geq V^g_{\pi^*_{\text{env}}, \mu} - C(\gamma, R_\text{max}) \alpha \epsilon.
\end{align*}
Chaining the inequalities yields:
\begin{align*}
    V^g_{\pi^*_\text{mix}, \mu} &\geq V^g_{\pi^*_\text{mix}, M_\alpha} - C(\gamma, R_\text{max}) \alpha \epsilon \\
    &\geq V^g_{\pi^*_\text{env}, M_\alpha} - C(\gamma, R_\text{max}) \alpha \epsilon \\
    &\geq (V^g_{\pi^*_\text{env}, \mu} - C(\gamma, R_\text{max}) \alpha \epsilon) - C(\gamma, R_\text{max}) \alpha \epsilon \\
    &= V^g_{\pi^*_\text{env}, \mu} - 2C(\gamma, R_\text{max}) \alpha \epsilon,
\end{align*}
with $V^g_{\pi^*_\text{mix}, \mu} \geq V^g_{\pi^*_\text{env}, \mu}$ when $\epsilon_f = 0$.
\end{proof}

\end{document}